\documentclass[journal,sort&compress]{IEEEtran}

\usepackage[utf8]{inputenc}
\usepackage{amsmath}
\usepackage{amssymb}
\usepackage{balance}
\usepackage{graphicx}
\usepackage{booktabs}
\usepackage{color}
\usepackage{cite}
\usepackage{algorithm}
\usepackage{algorithmic}
\usepackage[caption = false,font = footnotesize]{subfig}
\usepackage{float}

\newtheorem{theorem}{Theorem}
\newtheorem{lemma}[theorem]{Lemma}
\newtheorem{proposition}[theorem]{Proposition}

\newenvironment{proof}[1][Proof]{\textit{#1.} }{\qedsymbol}
\newcommand{\qedsymbol}{\hspace{\fill}\rule{1.5ex}{1.5ex}}

\newcommand{\change}[1]{#1}

\begin{document}

\title{Distributed \change{Training of} \\Graph Convolutional Networks\vspace{.3cm}}
\author{Simone Scardapane, Indro Spinelli,~\IEEEmembership{Student Member, IEEE}, and Paolo Di Lorenzo,~\IEEEmembership{Senior Member, IEEE}
\thanks{Corresponding author email: simone.scardapane@uniroma1.it}
\thanks{The authors are with the Department of Information Engineering, Electronics and Telecommunications (DIET), Sapienza University of Rome, Italy. Email: \{simone.scardapane,indro.spinelli,paolo.dilorenzo\}. This work was funded by ``Bandi di Ateneo per la Ricerca 2019'', Sapienza University of Rome.}\vspace{-.5cm}}

\markboth{Preprint published on IEEE Transaction on Signal and Information Processing over Networks}%
{Spinelli \MakeLowercase{\textit{et al.}}: Distributed training for GCNs}

\maketitle

\begin{abstract}
 The aim of this work is to develop a fully-distributed algorithmic framework for training graph convolutional networks (GCNs). The proposed method is able to exploit the meaningful relational structure of the input data, which are collected by a set of agents that communicate over a sparse network topology. After formulating the centralized GCN training problem, we first show how to make inference in a distributed scenario where the underlying data graph is split among different agents. Then, we propose a distributed gradient descent procedure to solve the GCN training problem. The resulting model distributes computation along three lines: during inference, during back-propagation, and during optimization. Convergence to stationary solutions of the GCN training problem is also established under mild conditions. Finally, we propose an optimization criterion to design the communication topology between agents in order to match with the graph describing data relationships. A wide set of numerical results validate our proposal. To the best of our knowledge, this is the first work combining graph convolutional neural networks with distributed optimization.
\end{abstract}

\begin{IEEEkeywords}
Graph convolutional networks, distributed optimization, networks, consensus.
\end{IEEEkeywords}

\section{Introduction}

Nowadays, graph-based data are pervasive, with applications ranging from social networks \cite{newman2002random} to epidemiology \cite{jombart2011reconstructing}, recommender systems \cite{berg2017graph}, cyber-security \cite{allamanis2017learning}, sensor networks \cite{romero2017kernel}, natural language processing \cite{bastings2017graph}, genomics \cite{lee2010discovering}, and many more. In fact, graphs are the natural format for any data where the available structure is given by pairwise relationships.

In the field of machine learning, graphs can be found in two varieties. In the case of \textit{graph models}, such as graph neural networks (GNNs) \cite{bruna2014spectral,romero2017kernel,kipf2017semi,bronstein2017geometric}, graph kernels \cite{kang2012fast,kondor2016multiscale}, or manifold regularization techniques \cite{belkin2006manifold,cai2010graph}, input data are provided in the form of a graph (such as relations among friends in a social network), and the model must exploit this relational structure when evaluating its output. For example, many classes of GNNs alternate local operations on the components of the graph with message-passing steps between neighboring nodes \cite{bronstein2017geometric}. 
In this case, edges represent \textit{semantic} connections between two data points, such as relations in a knowledge base or friendship in a social network.

Alternatively, in \textit{distributed optimization} of machine learning models, the graph represents the underlying physical communication pattern between agents, which possess only a partial view on the overall dataset and/or model. For convex models, notable examples include the alternating direction method of multipliers (ADMM) \cite{boyd2011distributed,ouyang2013stochastic}, diffusion optimization \cite{cattivelli2008diffusion}, and distributed sub-gradient methods \cite{nedic2009distributed}.
 More recently, several distributed optimization models have been proposed also in the non-convex case \cite{bianchi2012convergence,di2016next,vlaski2019distributed,vlaski2019distributed2,yang2016parallel}, with applications to machine learning \cite{scardapane2017framework,george2019distributed}. In this setting, the graph imposes stringent constraints on the way information can be propagated across the nodes, i.e., an operation is allowed only if its implementation can be done through message passing between two neighboring agents.
 
The majority of works on distributed optimization of machine learning models assumes that different examples (be they vectors, images, or similar) are independently distributed across physical agents, with no relation among them. In several situations, however, data has meaningful relational structure that can be exploited. In this context, a distributed optimization algorithm should be able to exploit both kinds of connections, physical and semantic. In this paper, we investigate a problem at the boundary of these two worlds. In particular, we assume that a graph is available representing similarity among data points, but these points are also physically distributed among multiple agents. A second \textit{communication} graph then represents feasible communication paths among the agents. 

\change{Our algorithm is tailored to situations in which data is naturally represented in terms of a graph, requiring graph neural network models for its processing. Furthermore, data are collected by a set of independent agents, and sharing local information with a central processor is either unfeasible or not economical/efficient, owing to the large size of the network and volume of data, time-varying network topology, energy constraints, and/or privacy issues.} For example, this is the case of sensor networks \cite{akyildiz2002wireless}, in which energy requirements restrict communication between sensors, and at the same time their spatial relation correlates their measurements. \change{A practical case is studied in Sec. \ref{subsec:traffic}, where we consider a network of distributed traffic sensors. We assume that the sensors send data to a fixed set of base stations (the \textit{agents} in our paper), which in turn communicate between them to perform the training of a graph neural network model to forecast traffic.} Similarly, we can think of medical entities possessing text databases, wherein distributed communication is used to ensure privacy, while the data graph can describe relations (e.g., citations) among texts and/or patients. \vspace{-.1cm}

\begin{figure*}
    \centering
    \includegraphics[width=0.97\textwidth]{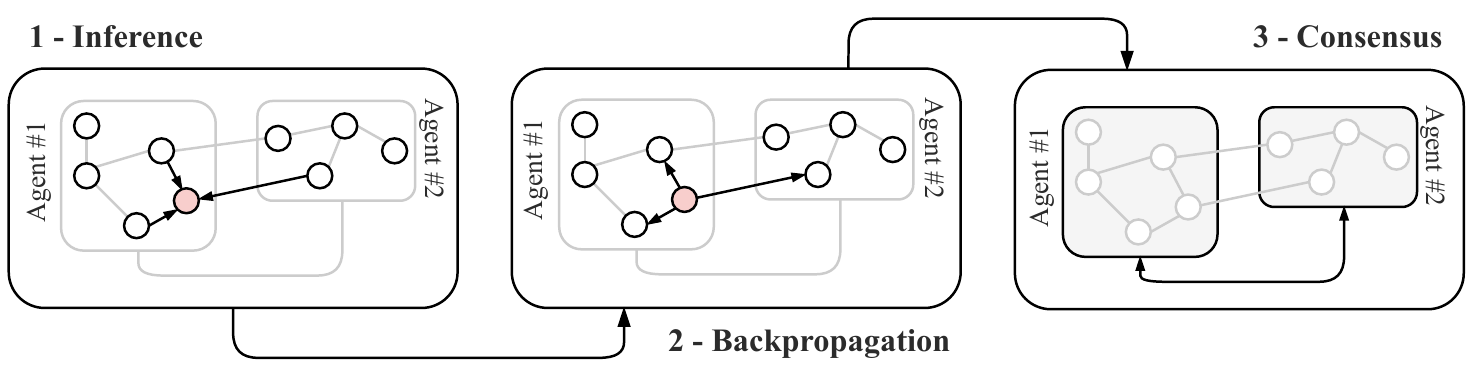}
    \caption{Illustration of the proposed approach. In step 1, nodes communicate to perform inference. In step 2, a symmetric communication phase is executed to compute local gradients. In step 3, agents exchange local variables to asymptotically reach agreement. For steps 1-2, a representative active node is shown in red. Directed arrows show the flow of messages.}
    \label{fig:proposed_approach}
\end{figure*}

\subsection*{Related works}
In this paragraph, we review the relevant literature on graph-based machine learning, optimization, and signal processing, considering separately the single research areas.

\subsubsection*{Distributed optimization}
Classical machine learning techniques assume that the training data is given by a set of independently distributed samples. In distributed optimization, it is further assumed that these samples are distributed across multiple agents having only sparse connectivity, either because of communication or privacy concerns \cite{sayed2014adaptation}. The objective is to devise globally convergent algorithms that can be implemented with local exchanges across agents. \change{As mentioned in the introduction, this class of approaches includes ADMM \cite{boyd2011distributed}, diffusion methods \cite{cattivelli2008diffusion}, non-convex distributed gradient descent \cite{bianchi2012convergence}, sub-gradient methods \cite{nedic2009distributed,nedic2010constrained}, projection-based methods \cite{di2020distributed,nassif2020adaptation,nassif2020adaptation2}, and several others, e.g., \cite{di2016next,vlaski2019distributed,vlaski2019distributed2}. Some works can also be found in the specific application of distributed optimization of neural networks \cite{scardapane2017framework}. Differently from this class of works, in this paper we assume that the data has an additional relational structure, in the form of pairwise connections between training points. Such structure typically induces a coupling among the objective functions of the agents, which must be properly taken into account in the development of tailored distributed optimization methods.} In fact, graph data has been receiving renewed interest in the machine learning literature, as detailed next.

\subsubsection*{Graph signal processing}
The field of graph signal processing (GSP) is motivated with extending classical signal processing tools to graph-based data \cite{ortega2018graph}. \change{Notable results include the definition of a graph Fourier transform \cite{ortega2018graph}, uncertainty principle \cite{tsitsvero2016signals},
graph filters \cite{sandryhaila2014discrete,isufi2016autoregressive,liu2018filter}, and adaptive reconstruction strategies \cite{di2018adaptive}. Closer to our work, several authors have considered completely distributed strategies for GSP techniques, including distributed recursive least-squares \cite{di2016distAdaGraph,di2018adaptive}, distributed graph filtering \cite{coutino2019advances}, diffusion adaptation for graph signals \cite{nassif2018distributed,hua2020online}, and also kernel-based approaches \cite{romero2017kernel}.} Linear and kernel methods, however, lack the flexibility in modelling nonlinear data with respect to modern deep neural networks. In addition, the majority of these works assume that the connections between data points exactly match the communication structure underlying the agents; an assumption that we relax in this paper.

\subsubsection*{Graph neural networks}
Similarly to GSP, the field of graph neural networks (GNNs) tries to extend the impressive results of deep neural networks to graph-based data. The literature on GNNs is vast and we refer to recent surveys for an overview \cite{bronstein2017geometric,zhou2018graph,wu2019comprehensive,bacciu2019gentle}. Most GNNs can be categorized in three broad classes: message-passing neural networks (MPNNs) \cite{micheli2009neural}, extensions of recurrent networks to graph-based data \cite{scarselli2008graph}, and non-linear extensions of GSP techniques, going under the broad name of geometric deep learning \cite{bronstein2017geometric}. The last class originates from \cite{bruna2014spectral}, which proposed to use graph Fourier transforms to replicate classical filtering operations used in convolutional neural networks (CNNs). \change{Depending on the type of graph filter used for the graph convolution, a number of GNNs with varying characteristics can be obtained \cite{defferrard2016convolutional,kipf2017semi,gama2020signals}. In the limiting case of a linear filter, the resulting network only requires messages with $1$-hop neighbors and can be implemented as a MPNN. The resulting order-1 GCN \cite{kipf2017semi} (or simply GCN for shortness in the paper) has gained widespread popularity, especially for semi-supervised learning on graphs.} While the output of a GNN (and the corresponding back-propagation rule) naturally decomposes over the nodes of the graph \cite{scarselli2008graph}, its computation requires that all nodes possess a shared copy of the GNN weights, which is not feasible in the distributed scenario considered here.

\change{\subsubsection*{Parallel and federated learning}
Several authors have considered \textit{parallel} implementations of GNNs \cite{lerer2019pytorch,zhu2019aligraph} to exploit multiple workers in parallel to optimize large-scale GCNs. The aim of these works, however, is only to scale and speed up training, with the possibility of freely exchanging data and parameters between the coordinator and the workers. More in general, federated learning (FL) \cite{konevcny2016federated,yang2019federated} considers the scenario where data belongs to multiple owners, with possibly strong privacy and/or communication constraints. The vast majority of FL works, however, considers the presence of one or more coordinators of the training process, while we work in a fully-distributed scenario. In addition, no author to the best of our knowledge has considered a scenario where both data and agents have a meaningful connectivity structure.
}

\subsection*{Contributions of the paper}

\change{In this paper, we propose the first algorithm for distributed training of GCNs in the scenario described above. We extend the basic graph convolutional network (GCN) model \cite{kipf2017semi}, incorporating ideas from
distributed optimization \cite{di2020distributed,nassif2020adaptation,nassif2020adaptation2} to perform inference and training using only message-passing exchanges among different agents. Interestingly, the resulting model is naturally distributed along three different lines: (i) during prediction, agents communicate among them according to the standard GCN model to compute their output; (ii) during backpropagation, a symmetric set of communication steps is executed to compute a local gradient for updating the model agent-wise; and (iii) a local consensus step is done to ensure convergence to a common architecture across the entire network. We refer to Fig. \ref{fig:proposed_approach} for a pictorial description of the approach. The convergence properties of the proposed strategy to stationary solutions of the GCN training problem are also analyzed in some detail. The method hinges on distributed projected gradient methods \cite{di2020distributed,nassif2020adaptation,nassif2020adaptation2}, extending them to the case of nonseparable objective functions, whose coupling comes from the relational structure given by the data graph.} 

In our formulation, we assume that the data graph is given, while the user has a certain flexibility in designing the connectivity between the physical agents. To this end, we also propose a simple (convex) criterion to optimize the connectivity between agents and ensure convergence of the distributed GCN training, given the underlying constraints arising from the data graph. The convex problem is solved using a low-complexity ADMM algorithm, which is applicable also in the case of large number of network agents.

We experimentally evaluate our distributed GCN on a number of benchmark datasets for graph neural network models, associating a data sub-graph to each different agent. We illustrate that our protocol can reach performance on par with a standard centralized GCN, and vastly superior to any neural network that does not exploit the relational structure in the data. Also, we evaluate the robustness of the model in the case where the communication graph between agents is not consistent with the structure of the data, and we show that it scales favorably even for sparse communication graphs. \change{Finally, we apply our method to a vehicular traffic prediction problem, using real data collected by a network of sensors located in the San Francisco area \cite{chen2001pems}.}

\subsubsection*{Organization of the paper}

In Sec. \ref{sec:graph_convolutional_networks} we introduce the basic GCN model. Sec. \ref{sec:distributed_training_for_gcns}, the main innovative contribution of the work, describes our distributed setup, the proposed training algorithm, and its convergence properties. In Sec. \ref{sec:optimizing_consensus_matrix}, we describe a criterion to optimize the agents' connectivity in order to match with the data graph. Finally, we numerically validate the proposed approach in Sec. \ref{sec:experimental_evaluation}, before drawing some conclusions in Sec. \ref{sec:conclusions_and_future_work}.

\subsubsection*{A note on indexing}
Throughout the paper, we consider two types of graphs, either related to the data (\textit{data graph}) or related to the connectivity of the physical agents (\textit{communication graph}). To simplify reading, we use $i$, $j$, $l$ to index nodes in the data graph, whereas $k$ and $z$ are used to index agents inside the communication graph.

\section{Graph convolutional networks}
\label{sec:graph_convolutional_networks}

\change{In this section, we introduce the GCN model \cite{wu2019comprehensive}. We focus on order-1 GCNs \cite{kipf2017semi} (corresponding to an order-1 FIR filter over the graph \cite{gama2020signals}) to enhance readability. Our framework can however be applied to any graph neural network that can be implemented via iterative message-passing over a graph, e.g., GAT \cite{velivckovic2017graph}, GINN \cite{xu2018powerful}, and several others. In particular, we show an extension to order-$P$ GCNs, equivalent to a generic FIR filter over the graph \cite{gama2020signals,isufi2020edgenets}, in Section \ref{subsec:extension_to_order_p}.}

Consider a generic \textit{data graph} $\mathcal{G}_D = \left(\mathcal{V}_D, \mathcal{E}_D\right)$, where $\mathcal{V}_D = \left\{1, \ldots, n\right\}$ is the set of $n$ vertexes (each corresponding to a data point), and $\mathcal{E}_D$ a set of edges between them. The edges may encode different kinds of relationships among the data as, e.g., similarities, correlations, causal dependencies, etc. The graph is equivalently described by a (weighted) matrix $\mathbf{D} \in \mathbb{R}^{n \times n}$ encoding adjacency information (i.e., a shift operator, in the GSP terminology). Every vertex is associated with a data vector $\mathbf{x}_i \in \mathbb{R}^d$ (e.g., a bag-of-words for a text, or a vector of features for a user in a social network). A subset $\mathcal{T}_D \subset \mathcal{V}_D$ of vertexes has associated a corresponding label $y_i$, $i \in \mathcal{T}_D$. The task of (semi-supervised) node classification aims at assigning a label also to the remaining set $\mathcal{V}_D \setminus \mathcal{T}_D$.

\change{An order-1 graph convolutional (GC) layer} is defined by (ignoring biases for simplicity, but w.l.o.g.):
\begin{equation}
    \mathbf{H} = \phi\left( \mathbf{D}\mathbf{X}\mathbf{W} \right) \,,
    \label{eq:gc_layer}
\end{equation}
where $\mathbf{X} \in \mathbb{R}^{n \times d}$ is a matrix collecting all vertex features row-wise, $\mathbf{W}\in\mathbb{R}^{d \times q}$ is a matrix of trainable coefficients, and $\phi$ is an element-wise non-linear function (e.g., for ReLU $\phi(s)=\max\left(0, s\right)$). The matrix $\mathbf{D}$ encodes edge information in the data graph, and can be chosen as any graph shift operator such as the Laplacian matrix, or normalized versions of it.

A graph convolutional network (GCN) $f$ is built by stacking multiple layers in the form \eqref{eq:gc_layer}. For example, a two-layered GCN for classification is described by:
\begin{equation}
    f(\mathbf{W},\mathbf{V};\mathbf{X}) = \text{softmax}\left( \mathbf{D} \, \phi\left( \mathbf{D}\mathbf{X}\mathbf{W}\right)\mathbf{V} \right) \,,
    \label{eq:two_layer_gcn}
\end{equation}
where $\mathbf{W}$ and $\mathbf{V}$ are the two matrices of adaptable coefficients. The generalization to more than two layers is straightforward. Note that the GCN in \eqref{eq:two_layer_gcn} acts on the data graph through the matrix $\mathbf{D}$. For instance, the output of \eqref{eq:two_layer_gcn} for a single data point depends on the $2$-hop neighbourhood of the node. In general, for a GCN with $K$ layers of the form \eqref{eq:gc_layer}, the output for node $i$ will depend on neighbours up to $K$ steps away. We denote this set by $\mathcal{N}_i^K$. Also, denote by $\hat{y}_i$ the prediction of the GCN on node $i$ (i.e., the $i$-th element of $f$), and by $\mathbf{w}$ the vector collecting all the trainable coefficients of $f$ [e.g., $\mathbf{w} = \left[ \text{vec}(\mathbf{W}); \text{vec}(\mathbf{V}) \right]$ in \eqref{eq:two_layer_gcn}]. Then, the optimal trained vector $\mathbf{w}^*$ of the GCN is obtained by solving the optimization problem:
\begin{equation}
    \mathbf{w}^* = \arg\min_{\mathbf{w}} \left\{ \frac{1}{\lvert \mathcal{T}_D \rvert}\sum_{i \in \mathcal{T}_D} l(\mathbf{w};\hat{y}_i,y_i) \right\} \,,
    \label{eq:opt}
\end{equation}
where $l(\mathbf{w};\cdot,\cdot)$ is a suitable loss function, e.g., squared loss for prediction, or cross-entropy for multi-class classification.  Problem (\ref{eq:opt}) is generally solved in a centralized manner by gradient descent or some accelerated first-order algorithm \cite{wu2019comprehensive}.

\section{Distributed Graph Convolutional Networks}
\label{sec:distributed_training_for_gcns}

In this section, we introduce our model for distributed GCNs, devising a simple, but very effective, algorithmic solution that performs inference and training using only message-passing exchanges in the neighbourhood of each agent.

\subsection{Problem setup}

Let us consider a setting in which the vertexes of the data graph $\mathcal{V}_D$ (i.e., the data points) are distributed across a set of physically separated agents (Fig. \ref{fig:proposed_approach}, left, where every white node represents a vertex in the data graph $\mathcal{V}_D$). We assume that the number $m$ of agents is much lower than the number $n$ of data points, i.e., $m \ll n$. The agents are physically distributed, and they can communicate among them according to a \textit{communication graph} (or agents' graph) $\mathcal{G}_C = \left(\mathcal{V}_C, \mathcal{E}_C\right)$. That is, agent $k$ can send messages to agent $z$ only if the link $(k,z) \in \mathcal{E}_C$.\footnote{Recall that we use indexes $k$ and $z$ to refer to an agent, while indexes $i$ and $j$ to refer to data nodes, like in the previous section.} No centralized coordinator is assumed to be available. Node $i$ in the data graph is assigned to agent $a(i) \in \mathcal{V}_C$, where we refer to $a(\cdot)$ as the assignment function that matches data points with agents. The data graph and assignment function are supposed to be given in the current problem setup. We also assume that, whenever there is an edge $(i, j) \in \mathcal{E}_D$ in the data graph, there is a corresponding viable communication path $(a(i), a(j)) \in \mathcal{E}_C$.\footnote{\change{Whenever this is not valid, we might remove the edge from the original data graph, eventually re-normalizing the weighted adjacency matrix. For some data graphs (e.g., similarity graphs) this is feasible, and we evaluate the effect on performance in Sec. \ref{subsec:results_with_sparse_connectivity}.}} 
Once this assumption is satisfied, the rest of the communication graph can be optimized to balance the amount of messages exchanged on the agents' network and the speed of convergence of our algorithm. In Section \ref{sec:optimizing_consensus_matrix}, we describe an optimization procedure to  select a connectivity matrix that satisfies both the agreement of the agents' graph with the data graph, and a user-selected level of global connectivity.

Before proceeding, we introduce some additional notation that will be helpful in the sequel. The set of training points available to agent $k$ is given by:
\begin{equation}
\mathcal{T}_k = \left\{ i \,\vert\, a(i) = k \right\} \,,
\label{eq:training_set_k}
\end{equation}
that is, training points assigned by $a(i)$ to agent $k$. Alternatively, we also use a shorthand $\mathbb{I}_k(a(i))$ to define an indicator function evaluating the proposition $a(i) = k$. Finally, we denote by $\mathcal{C}_i^K = \left\{ a(i) \; \vert \; i \in \mathcal{N}_i^K\right\}$ the set of agents having access to (at least one) of the $K$-hop neighbors of the data point $i$. We underline that $\mathcal{N}_i^K$ and $\mathcal{C}_i^K$ are different: the first one is a set of data points, whereas the latter is a set of agents. For instance, all the neighbors of datum $i$ can be on the same agent $k$, or they can be found across multiple agents (see Fig. \ref{fig:proposed_approach}). This makes the design of a fully-distributed strategy especially challenging due to the coupling between data and communication graphs. Naturally, we have that $\mathcal{C}_i^1 \subseteq \mathcal{C}_i^2 \subseteq \ldots \subseteq \mathcal{C}_i^K$.

\subsection{Distributed GCN Inference}

To devise a distributed GCN, we need to rewrite the basic layer \eqref{eq:gc_layer} in a way that is amenable for a fully-distributed implementation. Then, stacking multiple layers of this form will naturally provide us with a distributed GCN model. In particular, the operation of a GC layer is composed of a local component (the right-most multiplication by $\mathbf{W}$), and a communication step that is coherent with the communication topology under our assumption. However, the matrix $\mathbf{W}$ is a global parameter, and it is unrealistic to assume that all agents have access to it in the distributed scenario that we are considering in this work. Thus, we relax the model by assuming that each agent has its own set of weights $\mathbf{W}_k$, i.e., a local copy of the global variable $\mathbf{W}$ (see. Sec. III.C for distributed training of the local weights). In such a case, the $i$-th output row of the GC layer in (\ref{eq:gc_layer}) reads as:
\begin{equation}
    \mathbf{h}_i = \phi\left( \sum_{j \in \mathcal{N}_i} D_{ij} \Big[ \mathbf{W}_{a(j)}^T \mathbf{x}_j \Big] \right) \,,
    \label{eq:gc_layer_dist}
\end{equation}
where $\mathcal{N}_i$ is the (one-hop) inclusive neighborhood of vertex $i$. For a GC layer in the form \eqref{eq:gc_layer_dist}, the $i$-th output depends on the agents in the set $\mathcal{C}_i^1$. The question is: What messages should be sent across the agents' network for computing layer \eqref{eq:gc_layer_dist}? To evaluate $\eqref{eq:gc_layer_dist}$ at node $i$, the agent $k = a(i)$ must receive the following messages from agents $z \in \mathcal{C}_i^1$:
\begin{equation}
     \sum_{j \in \mathcal{N}_i} \mathbb{I}_z(a(j)) \cdot D_{ij}\left[ \mathbf{W}^T_z \mathbf{x}_j \right].
    \label{eq:messages_sent_for_inference}
\end{equation}
In particular, \change{denoting by $B_{kz} = \sum_{(i,j) \in \mathcal{E}_D} \mathbb{I}_k(a(i))\mathbb{I}_z(a(j))$ the number of data points that share a connection from agent $z$ to agent $k$} (e.g., $B_{12} = 2$ in Fig. \ref{fig:proposed_approach}), and assuming that the GC layer inputs and outputs have size $H$, the two agents $z$ and $k$ must exchange $B_{kz}H$ values for the inference phase.

Stacking multiple layers in the form \eqref{eq:gc_layer_dist}, we obtain a fully-distributed version of a GCN. For example, the $i$-th output of a two-layer GCN can be written as:
\begin{equation}
\hat{y}_i = \phi\left( \sum_{l \in \mathcal{N}_i} D_{il} \mathbf{V}_{a(l)}^T \phi\left( \sum_{j \in \mathcal{N}_l} D_{lj} \Big[ \mathbf{W}_{a(j)}^T \mathbf{x}_j \Big] \right) \right) \,.
\label{eq:two_layer_gcn_dist}
\end{equation}
The expression in \eqref{eq:two_layer_gcn_dist} depends on $\mathbf{V}_k$ for $k \in \mathcal{C}_i^1$, but also on $\mathbf{W}_k$ for $k \in \mathcal{C}_i^2$. Thus, increasing the number of layers of the GCN, we enlarge the set of multi-hop neighbors of node $i$ in the data graph, thus potentially increasing the number of agents participating in the prediction \eqref{eq:two_layer_gcn_dist}. Note that \eqref{eq:two_layer_gcn_dist} can be implemented by two sequential communication steps over the network of agents, both exploiting messages in the form \eqref{eq:messages_sent_for_inference}.

\subsection{Distributed GCN training}
\label{eq:distributed_gcn_training}

To train the GCN in a distributed fashion, we need to solve \eqref{eq:opt} exploiting only message-passing over the communication graph. Let $\mathbf{w}_k\in \mathbb{R}^p$ be the vector comprising all trainable parameters of agent $k$ (e.g., $\mathbf{w}_k = \left[ \text{vec}(\mathbf{W}_k); \text{vec}(\mathbf{V}_k) \right]$ in \eqref{eq:two_layer_gcn_dist}). Then, we reformulate \eqref{eq:opt} introducing a consensus constraint that forces all the local copies $\{\mathbf{w}_k\}_{k=1}^m$ to be equal:  
\begin{align}\label{eq:opt_dist}
    \underset{\mathbf{w}_1, \ldots, \mathbf{w}_m}{\arg\min} & \;\;\; \frac{1}{\lvert \mathcal{T_D}\rvert} \sum_{i \in \mathcal{T}_D} l(\mathbf{w}_1, \ldots, \mathbf{w}_m;\hat{y}_i, y_i) \\
    \text{s.t.} & \;\;\; \mathbf{w}_1 = \mathbf{w}_2 = \ldots = \mathbf{w}_m \,. \nonumber
\end{align}
\change{Inspired by the methods in \cite{di2020distributed,nassif2020adaptation,nassif2020adaptation2}}, we proceed minimizing \eqref{eq:opt_dist} by means of a distributed projected gradient algorithm, which we will explain in the sequel. 

Letting $\overline{\mathbf{w}}=[\mathbf{w}_1^T,\ldots,\mathbf{w}_m^T]^T$ be the stack of all agents' vectors, the agreement constraint set in \eqref{eq:opt_dist} can be equivalently rewritten as:
$$\mathcal{S}=\bigg\{\overline{\mathbf{w}}\;|\; \left(\frac{1}{m}\mathbf{1}_m\mathbf{1}_m^T \otimes \mathbf{I}_p\right)\overline{\mathbf{w}}=\overline{\mathbf{w}}\bigg\},$$ 
where $\otimes$ is the Kronecker product, $\mathbf{1}_m$ is the $m$-dimensional vector of all ones, and $\mathbf{I}_p$ is the $p$-dimensional identity matrix. In other words, the matrix 
\begin{equation}\label{Oth_proj_consensus}
\change{\mathbf{\Pi}_{\mathcal{S}}}=\frac{1}{m}\mathbf{1}_m\mathbf{1}_m^T \otimes \mathbf{I}_p
\end{equation}
is the orthogonal projection operator onto the consensus space. 
Then, letting $$L(\overline{\mathbf{w}})=\frac{1}{\lvert \mathcal{T}_D \rvert}\sum_{i \in \mathcal{T}_D} l(\overline{\mathbf{w}};\hat{y}_i, y_i),$$
a centralized projected gradient algorithm would proceed by iteratively minimizing \eqref{eq:opt_dist} as:
\begin{equation}\label{Cent_gradient_proj}
    \overline{\mathbf{w}}^{t+1} =  \mathbf{\Pi}_{\mathcal{S}}\big(\overline{\mathbf{w}}^t - \eta^t \nabla_{\overline{\mathbf{w}}} L(\overline{\mathbf{w}}^t) \big) \,, 
\end{equation}
where $\eta^t$ is a (possibly time-varying) step-size sequence. 

In principle, the procedure in \eqref{Cent_gradient_proj} can be implemented only by a central processor that has access to all the data. Indeed, in first place, the function $L(\overline{\mathbf{w}})$ is not separable among the agents's variables \change{(a key difference with respect to \cite{di2020distributed,nassif2020adaptation,nassif2020adaptation2})}. However, thanks to the specific coupling induced by the data graph, $\nabla_{\overline{\mathbf{w}}} L(\overline{\mathbf{w}}^t)$ in \eqref{Cent_gradient_proj} can still be computed in a distributed fashion. To see this, let $\nabla_{\mathbf{w}_k} L(\overline{\mathbf{w}}^t)$ be the gradient component of $\nabla_{\overline{\mathbf{w}}} L(\overline{\mathbf{w}}^t)$ associated with agent $k$. This $k$-th sub-component can be computed locally by exchanging data only within the neighborhood of agent $k$, with the aim of performing a distributed backpropagation step (cf. Fig. \ref{fig:proposed_approach}). In particular, considering a GCN with $K$ layers,  the set of all agents involved in the computation of $\nabla_{\mathbf{w}_k}L(\overline{\mathbf{w}}^t)$ is given by:
\begin{equation}
\overline{\mathcal{C}}_k = \bigcup_{i \in \mathcal{T}_k} \mathcal{C}_i^K \,,
\label{eq:C_definition}
\end{equation}
which means that
\begin{equation}\label{local_gradient}
\nabla_{\mathbf{w}_k}L(\overline{\mathbf{w}}^t)=\nabla_{\mathbf{w}_k} L(\{\mathbf{w}_z^t\}_{{z \in \overline{\mathcal{C}}_k}}).
\end{equation}
\change{The equality in \eqref{local_gradient} can be seen readily from the definition of $\overline{\mathcal{C}}_k$ in \eqref{eq:C_definition}. Consider for example an order-1 GC layer in the form \eqref{eq:gc_layer_dist}, and an agent $z \notin \overline{\mathcal{C}}_k$. Then, by the definition in \eqref{eq:training_set_k}, we have that $D_{ij} = 0$ in \eqref{eq:gc_layer_dist} and \eqref{eq:messages_sent_for_inference} whenever $a(i)=k$ and $a(j)=z$, so that no component in $L$ can depend on $\mathbf{w}_z^t$.} Similarly to the inference phase (see Sec. III.B), the computation of (\ref{local_gradient}) requires (possibly sequential) communication steps, where network agents send messages that are symmetric to those in \eqref{eq:messages_sent_for_inference}, as can be seen by explicitly differentiating \eqref{eq:two_layer_gcn_dist} [or (\ref{eq:gc_layer_dist})].
Note that, in this phase, nodes without labeled examples will act in a passive way, simply diffusing the required messages across the network.

Now, even if the gradient step can be computed in a distributed fashion, the projection step performed in \eqref{Cent_gradient_proj} through the multiplication by the matrix $\mathbf{\Pi}_{\mathcal{S}}$ is inherently a global averaging step from all agents. In the distributed setting that we consider in this work, there is no central agent that can perform such data fusion step. Thus, to find a distributed implementation capable of solving problem (\ref{eq:opt_dist}), we substitute the projection operator $\mathbf{\Pi}_{\mathcal{S}}$ in \eqref{Cent_gradient_proj} with a sparse combination matrix $\mathbf{C}$, whose sparsity pattern reflects the topology of the communication graph. In the sequel, we illustrate the properties of the combination matrix $\mathbf{C}$.

\textit{Assumption A:} The matrix $\mathbf{C}\in\mathbb{R}^{m \times m}$ satisfies the following properties:
\begin{description}
  \item[(A1)] The matrix is symmetric  and doubly stochastic, i.e., $\mathbf{1}^T\mathbf{C}=\mathbf{1}^T$ and $\mathbf{C}\mathbf{1}=\mathbf{1}$;\smallskip
  \item[(A2)]   
  $\rho\left(\mathbf{C}-\displaystyle\frac{1}{m}\mathbf{1}_m\mathbf{1}_m^T\right)\leq 1-\gamma<1$, where $\rho(\mathbf{X})$ denotes the spectral radius of matrix $\mathbf{X}$. 
  \change{\item[(A3)] $C_{kz}= 0$ if the edge $(k,z) \notin \mathcal{E}_C$, for all $(k,z)$, i.e., the matrix reflects the connectivity between agents.}
\end{description}
Assumption (A2) ensures that the communication graph is connected, and information can propagate all over the network. \change{Indeed, w.l.o.g., any combination matrix can be written as $\mathbf{C}=\mathbf{I}-\mathbf{L}$, where $\mathbf{L}$ is a weighted Laplacian matrix. By construction, $\mathbf{C}$ has a constant eigenvector associated with the unitary eigenvalue, whose multiplicity is equal to the number of connected components of the graph \cite{di2020distributed}. Thus, A2 states that, once we remove the constant eigenvector (i.e., the consensus subspace) from $\mathbf{C}$, the spectral radius of the resulting matrix is strictly less than one, thus ensuring graph connectivity and stability of the consensus recursion \cite{nedic2009distributed}.} Then, under (A1) and (A2), the matrix $\mathbf{C}$ satisfies:
\begin{equation}
    \lim_{t\rightarrow \infty } \mathbf{C}^t = \frac{1}{m}\mathbf{1}_m\mathbf{1}_m^T,
\end{equation}
i.e., the iterative application of the combination matrix $\mathbf{C}$ performs in the limit the orthogonal projection onto the agreement set $\mathcal{S}$. In conclusion, the proposed distributed training algorithm has the following structure:
\begin{equation}\label{Dist_gradient_proj}
    \overline{\mathbf{w}}^{t+1} =  \left(\mathbf{C} \otimes \mathbf{I}_p\right)\bigg(\overline{\mathbf{w}}^t - \eta^t \nabla_{\overline{\mathbf{w}}} L(\overline{\mathbf{w}}^t) \bigg) \,. 
\end{equation}
Recursion (\ref{Dist_gradient_proj}) can now be implemented in a fully-distributed fashion. Then, exploiting (\ref{local_gradient}), and letting 
\begin{equation}
\boldsymbol{\psi}_k^t=\mathbf{w}_k^t-\eta^t \nabla_{\mathbf{w}_k} L(\{\mathbf{w}_z^t\}_{{z \in \overline{\mathcal{C}}_k}})
\label{eq:dist_gd}
\end{equation}
be the result of the $k$-th component of the gradient step in (\ref{Dist_gradient_proj}), the update of the local estimate at agent $k$ is given by:
\begin{equation}
    \mathbf{w}^{t+1}_k = \sum_{z \in \mathcal{N}_k} C_{kz}\boldsymbol{\psi}_z^t \,,
    \label{eq:consensus_step}
\end{equation}
where $\mathcal{N}_k$ is the inclusive neighborhood of agent $k$ in the communication graph (including agent $k$ itself). 

The main steps of the proposed distributed training strategy for GCNs are summarized in Algorithm 1.

\begin{algorithm}[t]
\caption{\textbf{: Distributed GCN training and inference}}
\label{algo:distributed_gcn_training}
\vspace{.1cm}
\textbf{Data:} Number of iterations $T$; matrix $\mathbf{C}$ satisfying (A1)-(A3); step-size sequence $\eta^t$.
\begin{algorithmic}[1]
%\STATE Compute $\mathbf{C}$ with Algorithm \ref{algo:admm}.
\STATE Randomly initialize $\mathbf{w}_1^0, \ldots, \mathbf{w}_m^0$.
\FOR{$t \in \left\{0, \ldots, T-1\right\}$, all agents in parallel}
    \STATE  Evaluate the prediction $\hat{y}_i$ for all labeled data.
    %\STATE  Evaluate $\nabla_{\mathbf{w}_k} L(\{\mathbf{w}_z^t\}_{{z \in \overline{\mathcal{C}}_k}})$.
    \STATE  Compute the local estimate $\boldsymbol{\psi}_k^{t}$ using \eqref{eq:dist_gd}.
    \STATE  $\mathbf{w}^{t+1}_k = \sum_{z \in \mathcal{N}_k} C_{kz}\boldsymbol{\psi}_z^t$
\ENDFOR
\end{algorithmic}
\end{algorithm}

\textbf{A note on privacy}: In a distributed optimization context, it is generally required that the algorithm is privacy-preserving, i.e., no input data is ever transparently exposed over the network. While the consensus step in \eqref{eq:consensus_step} trivially ensures this property, this is slightly more critical for the inference of the first layer, in which the agents exchange linear combinations of the data points. While the matrices $\mathbf{W}_k$ are not necessarily invertible, this could still pose a non-negligible privacy threat whenever the two agents share a large number of connections. We leave an investigation of differentially-private \cite{abadi2016deep} versions of this algorithm for a future work.

\subsection{Convergence Analysis}
\label{eq:convergence_analysis}

In  this  section,  we  illustrate  the  convergence of the proposed distributed training algorithm to stationary solutions of \eqref{eq:opt}.

\begin{proposition} A point $\overline{\mathbf{w}}^*\in \mathcal{S}$ is a stationary solution of Problem (\ref{eq:opt_dist}) if a gradient $\nabla L(\overline{\mathbf{w}}^*)$ exists such that \footnote{To ease the notation, from now on we will use $\nabla L$ instead of $\nabla_{\overline{\mathbf{w}}} L$.}
\begin{align}\label{stat_point}
\mathbf{\Pi}_{\mathcal{S}} \nabla L(\overline{\mathbf{w}}^*)= \mathbf{0}. \smallskip
\end{align}
\end{proposition}
\begin{proof}
From the minimum principle, any stationary solution $\overline{\mathbf{w}}^*$ of (\ref{eq:opt_dist}) must satisfy
\begin{align}\label{min_princ}
\overline{\mathbf{w}}^*=\mathbf{\Pi}_{\mathcal{S}}\big(\overline{\mathbf{w}}^*-\nabla L(\overline{\mathbf{w}}^*)\big).
\end{align}
Since $\mathbf{\Pi}_{\mathcal{S}}\overline{\mathbf{w}}^*=\overline{\mathbf{w}}^*$, (\ref{min_princ}) leads immediately to (\ref{stat_point}). \smallskip
\end{proof}

We consider the following assumptions on function $L(\overline{\mathbf{w}})=\sum_{i \in \mathcal{T}_D} l(\overline{\mathbf{w}};\hat{y}_i, y_i)$ in problem (\ref{eq:opt_dist}).\smallskip

\noindent {\textit{Assumption B [On the objective function]}}: The objective function $L$ in (\ref{eq:opt_dist}) is continuous and satisfies:
\begin{description}
  \item[(B1)] $\,L$ is a differentiable, nonconvex function, with Lipschitz continuous gradient, i.e.,
\begin{align}
\|\nabla L(\mathbf{w})-\nabla L(\mathbf{z})\|\leq c_L\|\mathbf{w}-\mathbf{z}\|,\;\; \hbox{for all $\mathbf{w}$, $\mathbf{z}$}; \nonumber
\end{align}
  \item[(B2)] $L$ has bounded gradients, i.e., there exists $G>0$ such that $\|\nabla L(\mathbf{w})\|\leq G$ for all $\mathbf{w}$;
  \item[(B3)] $L$ is coercive, i.e., $\displaystyle \lim_{\|\mathbf{w}\|\rightarrow\infty} L(\mathbf{w})=+\infty.$
\end{description}

Assumption B is standard and satisfied by many practical problems, \change{see, e.g., \cite{nedic2009distributed,bianchi2012convergence,di2016next,george2019distributed}.}
In particular, assumption  (B3),  together  with  the  continuity  of $L$, guarantees the existence of a global minimizer of problem (\ref{eq:opt_dist}). Also, we consider two alternative choices for the step-size sequence $\{\eta^t\}_t$ in Algorithm 1, which are illustrated in the following assumption.\vspace{.1cm}

\noindent {\textit{Assumption C [On the step-size]}}: The step-size sequence $\{\eta^t\}_t$ is chosen as:
\begin{description}
  \item[(C1)] a constant, i.e., $\eta^t=\eta>0$ for all $t$;
  \item[(C2)] a diminishing sequence such that $\eta^t>0$, for all $t$,
\begin{equation}\label{step-size}
   \hbox{ $\displaystyle\sum_{t=0}^{\infty}\eta^t=\infty$ \hspace{.2cm} and \hspace{.2cm} $\displaystyle\sum_{t=0}^{\infty}(\eta^t)^2<\infty.$} \nonumber
\end{equation}
\end{description}
\change{Now, letting $\overline{\mathbf{w}}_{\mathcal{S}}^t=\mathbf{\Pi}_{\mathcal{S}}\overline{\mathbf{w}}^t$ be the orthogonal projection of the sequence $\{\overline{\mathbf{w}}^t\}_t$ onto $\mathcal{S}$}, we define the performance metric 
\change{\begin{equation}\label{g_nonconvex}
g^t=\|\nabla L(\overline{\mathbf{w}}_{\mathcal{S}}^t)\|^2_{\mathbf{\Pi}_{\mathcal{S}}}=\nabla L(\overline{\mathbf{w}}_{\mathcal{S}}^t)^T\mathbf{\Pi}_{\mathcal{S}}\nabla L(\overline{\mathbf{w}}_{\mathcal{S}}^t),
\end{equation}}
which quantifies proximity to a stationary solution of (\ref{eq:opt_dist}) [cf. (\ref{stat_point})]. 
Also, let  
\begin{equation}\label{g_best}
g_{best}^t= \inf_{n=1,\ldots,t} g^n = \inf_{n=1,\ldots,t}\|\nabla  L(\overline{\mathbf{w}}_{\mathcal{S}}^n)\|^2_{\mathbf{\Pi}_{\mathcal{S}}}.
\end{equation}

We now illustrate the convergence properties of the proposed distributed training strategy, which are summarized in the following Theorem.

\begin{theorem}\label{convergence_th}
Let $\{\overline{\mathbf{w}}^t\}_t$ be the sequence generated by Algorithm 1, and let $\overline{\mathbf{w}}_{\mathcal{S}}^t=\mathbf{\Pi}_{\mathcal{S}}\overline{\mathbf{w}}^t$ be its orthogonal projection onto $\mathcal{S}$. Suppose that conditions (A1)-(A3) and (B1)-(B3) hold. Then, the following results hold.
\begin{description}
  \item[(a)] \emph{\texttt{[Consensus]}}: Under (C1), we have:
\begin{equation}\label{sub_lim1}
\lim_{t\rightarrow\infty}\;\|\overline{\mathbf{w}}^t-\overline{\mathbf{w}}_{\mathcal{S}}^t \|= O(\eta);
\end{equation}
if (C2) holds, the sequence $\{\overline{\mathbf{w}}^t\}_t$  asymptotically converges to the consensus subspace $\mathcal{S}$, i.e.,
\begin{equation}\label{sub_lim2}
\lim_{t\rightarrow\infty}\;\|\overline{\mathbf{w}}^t-\overline{\mathbf{w}}_{\mathcal{S}}^t\|= 0;
\end{equation}

\item[(b)] \emph{\texttt{[Convergence]}}: Under (C1), we obtain:

\begin{equation}\label{conv_nonconvex1}
\lim_{t\rightarrow\infty}\; g_{best}^t= \mathcal{O}(\eta),
\end{equation}
with convergence rate $\mathcal{O}\left(\frac{1}{t+1}\right)$; finally, if (B2) holds, we have:
\begin{equation}\label{conv_nonconvex2}
\lim_{t\rightarrow\infty}\; g^t=0,
\end{equation}
i.e.,  $\{\overline{\mathbf{w}}_{\mathcal{S}}^t\}_t$  converges to a stationary solution of (\ref{eq:opt_dist}). \vspace{.1cm}
\end{description}
\begin{proof}
See Appendix B. 
\end{proof}

%and $\{\overline{\bx}[k]\}_k\triangleq \{\mathcal{P}_{\mathcal{R}(\mathbf{U})}\bx[k]\}_k$ be its projection onto the subspace $\mathcal{R}(\mathbf{U})$; 
%let also  $\{\bx_\perp[k]\}_k=\{\bx[k]-\overline{\bx}[k]\}_k$. Suppose that conditions (A2), (A3) and (C1)-(C3) hold. Then, the following results hold.

\end{theorem}

\change{
\subsection{Extension to order-$P$ GCN}
\label{subsec:extension_to_order_p}
The GCNs we considered up to now can be interpreted as a sequence of order-1 filters, interleaved with pointwise non-linearities. More generic order-$P$ GC layers generalize this setup, and they can be shown to outperform their order-1 counterparts \cite{shchur2018pitfalls,dwivedi2020benchmarking,isufi2020edgenets}. Extending \eqref{eq:gc_layer}, an order-$P$ GC layer can be written as \cite{isufi2020edgenets}:
\begin{equation}
    \mathbf{H} = \phi\left( \sum_{p=0}^P \mathbf{D}^p\mathbf{X}\mathbf{W}_{p} \right) \,,
    \label{eq:order_p_gc_layer}
\end{equation}
where $P$ is a hyper-parameter, and $\mathbf{W}_0, \ldots, \mathbf{W}_P$ are separate trainable parameters. While the $i$th output in \eqref{eq:order_p_gc_layer} depends on neighbors up to $P$ hops away, the layer is still implementable as $P$ iterative rounds of message passing across the graph. As a result, it is immediate to extend our setup to include them. Consider for example a distributed version of an order-2 GC layer, where agent $z$ is endowed with three local sets of parameters $\mathbf{W}_{p,z}$, $p=0,1,2$. Define for simplicity $\tilde{\mathbf{x}}_{p,i} = \mathbf{W}_{p,a(i)}^T\mathbf{x}_i$. Then, the distributed layer can be written as:
\begin{eqnarray}
    \mathbf{h}_i & = \phi\left( \tilde{\mathbf{x}}_{0,i} + \displaystyle\sum_{j \in \mathcal{N}_i} D_{ij}
     \Bigg[ \tilde{\mathbf{x}}_{1,j} + \displaystyle\sum_{l \in \mathcal{N}_j} D_{jl} \tilde{\mathbf{x}}_{2,l} \Bigg] \right) \,,
    \label{eq:order_2_gc_layer_dist}
\end{eqnarray}
where $\tilde{\mathbf{x}}_{1,j}$ and $\tilde{\mathbf{x}}_{2,l}$ are the messages to be exchanged across the graph. Similarly, we can concatenate multiple layers of the form \eqref{eq:order_2_gc_layer_dist} to obtain distributed order-2 GCNs. In this case, the $i$th output of $L$ order-$P$ layers will depend on the set $\mathcal{C}_i^{PL}$.
}

\section{Matching data and communication graphs}
\label{sec:optimizing_consensus_matrix}

To implement Algorithm 1, the communication network between the agents (described by $\mathbf{C}$) must satisfy three constraints: (i) it must be connected (even when the underlying data graph is disconnected) [cf. (A2)]; (b) the matrix $\mathbf{C}$ must be symmetric and doubly stochastic [cf. (A1)]; and (c) the graph must allow communication between two agents whenever they share at least one edge in the data graph, in order to enable the proposed distributed inference and training.

Recall from Section \ref{sec:distributed_training_for_gcns} that we denote with $B_{kz}$ the number of data edges in common between agents $k$ and $z$. Then, let us denote by $\mathbf{A} \in \left\{0,1\right\}^{m \times m}$ a Boolean matrix such that:
\begin{equation}
    A_{kz} = \begin{cases}
    1 & \text{ if } B_{kz} = 0; \\
    0 & \text{ otherwise, and if $k=z$. }
    \end{cases}
\end{equation}
Thus, the set of all feasible matrices $\mathbf{C}$ is described by:
\begin{align}
    \mathcal{X} = \Bigg\{ &\mathbf{C} \;\vert\; \mathbf{C}\mathbf{1} = \mathbf{1}, \, \mathbf{C}^T = \mathbf{C}, \,  \rho\left(\mathbf{C}-\displaystyle\frac{1}{m}\mathbf{1}_m\mathbf{1}_m^T\right)<1,  \nonumber\\ 
    &\hspace{.5cm} C_{kz} \neq 0 \text{ if } A_{kz} = 0 \Bigg\} \,.
    \label{eq:C_opt_set}
\end{align}
We can identify two extremes in the set $\mathcal{C}$. On one hand, we can connect only agents that share at least one connection between data points (the sparsest possible connectivity). On the other hand, we can consider fully-connected networks, where every pair of agents is connected. Of course, connectivity has also an impact on the convergence speed of the proposed distributed procedure. Thus, in general, we would like a principled way to optimize the communication topology given the constraints on the data graph and a target convergence speed of the consensus mechanism in (\ref{eq:consensus_step}). To this aim, we propose the following optimization problem to design the matrix $\mathbf{C}$ (and, consequently, the communication topology of the agents' network):
\begin{align}\label{eq:C_opt}
    \underset{\mathbf{C}}{\arg\min} & \;\;\; \lVert \mathbf{C} \odot \mathbf{A} \rVert_1  \nonumber\\
    \text{s.t.} & \;\;\; \mathbf{C} = \mathbf{C}^T,\;\;\; \mathbf{C}\mathbf{1} = \mathbf{1} \,, \\
                & \;\;\; \rho\left( \mathbf{C} - \frac{1}{m}\mathbf{1}\mathbf{1}^T \right) \le 1 - \gamma \,\nonumber 
\end{align}
where $\odot$ denotes the Hadamard product, $\lVert \cdot \rVert_1$ is the $\ell_1$ norm on the elements of the matrix, and $\gamma$ is a hyper-parameter to be chosen. The rationale underlying (\ref{eq:C_opt}) is to impose sparsity on the links that are not necessary to train the distributed GCN model (i.e., those such that $A_{kz}=1$ for all $k\neq z$), while enforcing conditions (A1) and (A2) on matrix $\mathbf{C}$. In particular, the last constraint in (\ref{eq:C_opt}) ensures the network connectivity, and controls the convergence speed of the consensus step in (\ref{eq:consensus_step}). In particular, lower values of $1-\gamma$ lead to larger graph connectivity and, consequently, to faster diffusion of information over the network. Thus, solving (\ref{eq:C_opt}) while varying $\gamma$ enables us to explore multiple feasible connectivities between the two extremes discussed above, while guaranteeing convergence of the proposed algorithm for distributed GCN training. Note that this problem must be solved only once, before deploying the distributed GCN model.

Problem (\ref{eq:C_opt}) is convex, and can be solved using standard numerical tools based on semi-definite programming (SDP) \cite{boyd2004convex}. However, it is known that such methods might become extremely inefficient if the number of variables (in this case, the number $m$ of agents) becomes too large, e.g., of the order of hundreds of agents. Thus, in the sequel, we propose an efficient iterative method based on the alternating direction method of multipliers (ADMM) \cite{boyd2011distributed} to solve (\ref{eq:C_opt}).

\subsection*{ADMM solution}

Let us denote the feasible set of (\ref{eq:C_opt}) as:
\begin{align}
    \widetilde{\mathcal{X}} = \Bigg\{ &\mathbf{C} \;\vert\; \mathbf{C}\mathbf{1} = \mathbf{1}, \, \mathbf{C}^T = \mathbf{C}, \,  \rho\left(\mathbf{C}-\displaystyle\frac{1}{m}\mathbf{1}_m\mathbf{1}_m^T\right)\leq 1-\gamma\Bigg\} \,. \nonumber
\end{align}
Then, let us consider the equivalent problem formulation of (\ref{eq:C_opt}), which writes as:
\begin{align}
    \underset{\mathbf{C}, \mathbf{Z}}{\arg\min} & \;\;\; \lVert \mathbf{Z} \odot \mathbf{A} \rVert_1 + \mathbb{I}_{\widetilde{\mathcal{X}} }(\mathbf{C}) \,, \label{eq:C_opt_relaxed} \\
    \hspace{.6cm}\text{s.t.} \;\;\; & \mathbf{C} = \mathbf{Z}   \nonumber\,
\end{align}
where $\mathbb{I}_{\widetilde{\mathcal{X}}}(\cdot)$ is the indicator function with respect to the feasible constraint set $\widetilde{\mathcal{X}}$. Now, we exploit the scaled version of the ADMM algorithm \cite{boyd2011distributed}, which alternates among the following three optimization steps. First, letting $\mathbf{U}^t$ be the estimate at time $t$ of the (scaled) Lagrange multiplier associated with the equality constraint in (\ref{eq:C_opt_relaxed}), we solve a regularized problem with respect to $\mathbf{C}$, given by:
\begin{align}
    \mathbf{C}^{t+1} &= \underset{\mathbf{C}}{\arg\min}  \left\{\mathbb{I}_{\widetilde{\mathcal{X}}}(\mathbf{C}) + \frac{1}{2}\lVert \mathbf{C} - \mathbf{Z}^t + \mathbf{U}^t \rVert_F^2 \right\} \,.
    \label{eq:admm_1}
\end{align}
The step in (\ref{eq:admm_1}) can be computed in closed-form as:
\begin{equation}
    \mathbf{C}^{t+1} = \mathcal{P}_{\widetilde{\mathcal{X}}}\Big( \mathbf{Z}^t - \mathbf{U}^t \Big) \,,
    \label{eq:admm_1_solved}
\end{equation}
where $\mathcal{P}_{\widetilde{\mathcal{X}}}(\cdot)$ denotes the orthogonal projection onto the feasible set $\widetilde{\mathcal{X}}$.
Following similar arguments as in \cite{nassif2020adaptation}, letting $\mathbf{\Pi}_{\mathbf{1}}=\mathbf{1}_m\mathbf{1}_m^T/m$, the projection of a generic matrix $\mathbf{X}$ onto the set $\widetilde{\mathcal{X}}$ is given by:
\begin{equation}\label{eq:projection}
    \mathcal{P}_{\widetilde{\mathcal{X}}}(\mathbf{X}) = \mathbf{\Pi}_{\mathbf{1}} + \sum_{i=1}^m \beta_i \mathbf{v}_i\mathbf{v}_i^T \,,
\end{equation}
where 
\begin{equation}\label{eq:eigenvalues}
    \beta_i = \begin{cases}
        -1+\gamma, & \text{ if } \lambda_i < -1+\gamma; \\
        1-\gamma, & \text{ if } \lambda_i > 1-\gamma; \\
        \lambda_i, & \text{ if } |\lambda_i|\leq 1-\gamma.
    \end{cases}
\end{equation}
and $\left\{\lambda_i, \mathbf{v}_i\right\}$ denote the eigenvalues and eigenvectors of matrix
\begin{equation}\label{eq:deflated_matrix}
 \left(\mathbf{I} - \mathbf{\Pi}_{\mathbf{1}}\right)\left(\frac{\mathbf{X} + \mathbf{X}^T}{2}\right) \left(\mathbf{I} - \mathbf{\Pi}_{\mathbf{1}}\right) \,.
\end{equation}

\begin{algorithm}[t]
\caption{\textbf{: ADMM-based solution for (\ref{eq:C_opt})}}
\vspace{.1cm}
\textbf{Data:} matrix $\mathbf{A}$; $\mathbf{C}^{0}$, $\mathbf{U}^{0}$, and $\mathbf{Z}^{0}$ chosen at random; $\varrho>0$. Then, for each time $t\geq0$, repeat the following steps:
\begin{align}
    &\hbox{\textbf{(S.1)}}\quad\mathbf{C}^{t+1} = \mathcal{P}_{\widetilde{\mathcal{X}}}\Big( \mathbf{Z}^t - \mathbf{U}^t \Big) \nonumber\\
    &\hspace{1cm}\hbox{where the projector   $P_{\widetilde{\mathcal{X}}}(\cdot)$ is given by (\ref{eq:projection})-(\ref{eq:deflated_matrix});}    \nonumber\\[3pt]
    &\hbox{\textbf{(S.2)}}\quad\mathbf{Z}^{t+1} \,=\, \mathcal{S}_{\frac{1}{\varrho}}\Big( \left(\mathbf{C}^{t+1} + \mathbf{U}^t\right) \odot \mathbf{A} \Big)\nonumber\\
    &\hspace{2.2cm}+ \left(\mathbf{C}^{t+1} + \mathbf{U}^t\right) \odot \left(\mathbf{1}_{m\times m} - \mathbf{A}\right)  \nonumber\\[5pt]
    &\hbox{\textbf{(S.3)}}\quad\mathbf{U}^{t+1} = \mathbf{U}^t + \mathbf{C}^{t+1} - \mathbf{Z}^{t+1}  \nonumber
\end{align}
\label{algo:admm}
\end{algorithm}

The second step of scaled ADMM performs the following optimization with respect to the primal variable $\mathbf{Z}$:
\begin{equation}
    \mathbf{Z}^{t+1} = \underset{\mathbf{Z}}{\arg\min} \left\{ \lVert \mathbf{Z} \odot \mathbf{A} \rVert_1 + \frac{\varrho}{2}\lVert \mathbf{C}^{t+1} - \mathbf{Z} + \mathbf{U}^t \rVert_F^2 \right\} \,,
    \label{eq:admm_2}
\end{equation}
where $\varrho>0$ is the penalty parameter of the augmented Lagrangian in ADMM. By setting the subgradient of the objective in (\ref{eq:admm_2}) to zero, it is straightforward to see that the solution is given by:
\begin{align}
    \mathbf{Z}^{t+1} \,=\, &  \mathcal{S}_{1/\varrho}\Big( \left(\mathbf{C}^{t+1} + \mathbf{U}^t\right) \odot \mathbf{A} \Big)\nonumber\\
    &+ \left(\mathbf{C}^{t+1} + \mathbf{U}^t\right) \odot \left(\mathbf{1}_{m\times m} - \mathbf{A}\right) \,,
    \label{eq:admm_2_solved}
\end{align}
where $\mathcal{S}_{\epsilon}(x)= {\rm sign}(x)\cdot\max (|x|-\epsilon,0)$ is the (element-wise) soft-thresholding operator with threshold parameter $\epsilon$. Finally, we perform the multipliers update of scaled ADMM:
\begin{equation}
    \mathbf{U}^{t+1} = \mathbf{U}^t +  \mathbf{C}^{t+1} - \mathbf{Z}^{t+1} \,.
    \label{eq:admm_3}
\end{equation}
The main steps of the proposed ADMM solution are summarized in Algorithm 2. The computational complexity of Algorithm 2 is determined by the projection step in (S.1), which requires $O(m^3)$ operations.

As a numerical example, we show in Fig. \ref{fig:w_opti} the sparsity of the matrix $\mathbf{C}$ (defined as the percentage of zero entries) with respect to $1-\gamma$, when solving the optimization problem (\ref{eq:C_opt}) for $m=20$ agents and randomly initialized matrix $\mathbf{A}$. From Fig. \ref{fig:w_opti}, we can notice how the sparsity of the communication graph improves by increasing the value of the hyper-parameter $1-\gamma$, until it reaches the minimally-connected network corresponding to the perfect matching with the data graph (i.e., $C_{kz} > 0$ i.f.f. $A_{kz} =0$).

\begin{figure}
	\centering
	\includegraphics[width=0.45\textwidth]{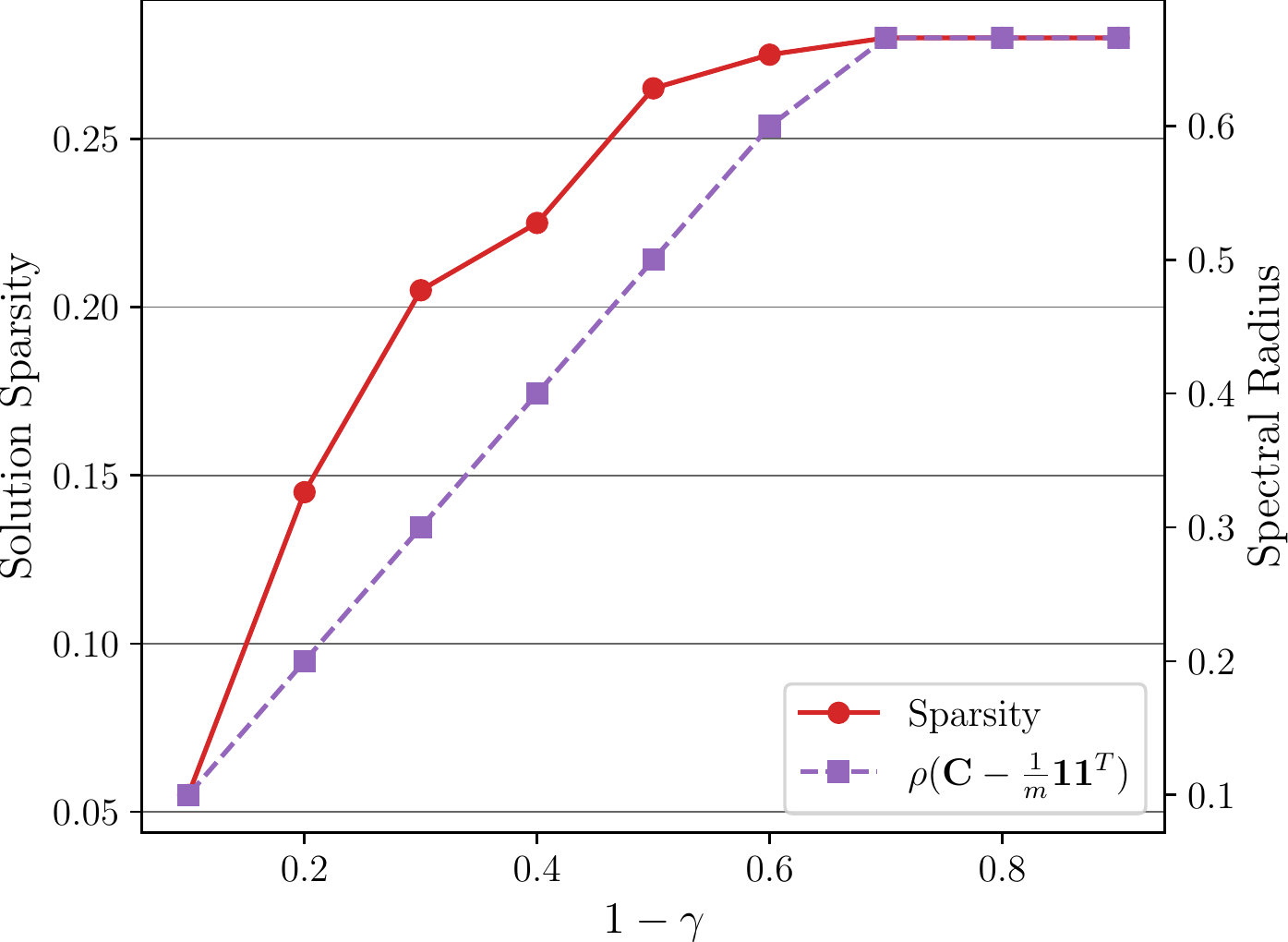}
	\caption{Sparsity of the final solution (percentage of zero entries) when solving the optimization problem for a varying $\gamma$.}
	\label{fig:w_opti}
\end{figure}

\section{Experimental evaluation}
\label{sec:experimental_evaluation}

In this section, we extensively assess the performance of the proposed strategy for distributed training (and inference) of graph convolutional networks.

\subsection{Experimental setup}
\label{subsec:experimental_setup}

\begin{figure*}[t]
\subfloat[CORA (loss)]{
\includegraphics[width=0.63\columnwidth,keepaspectratio]{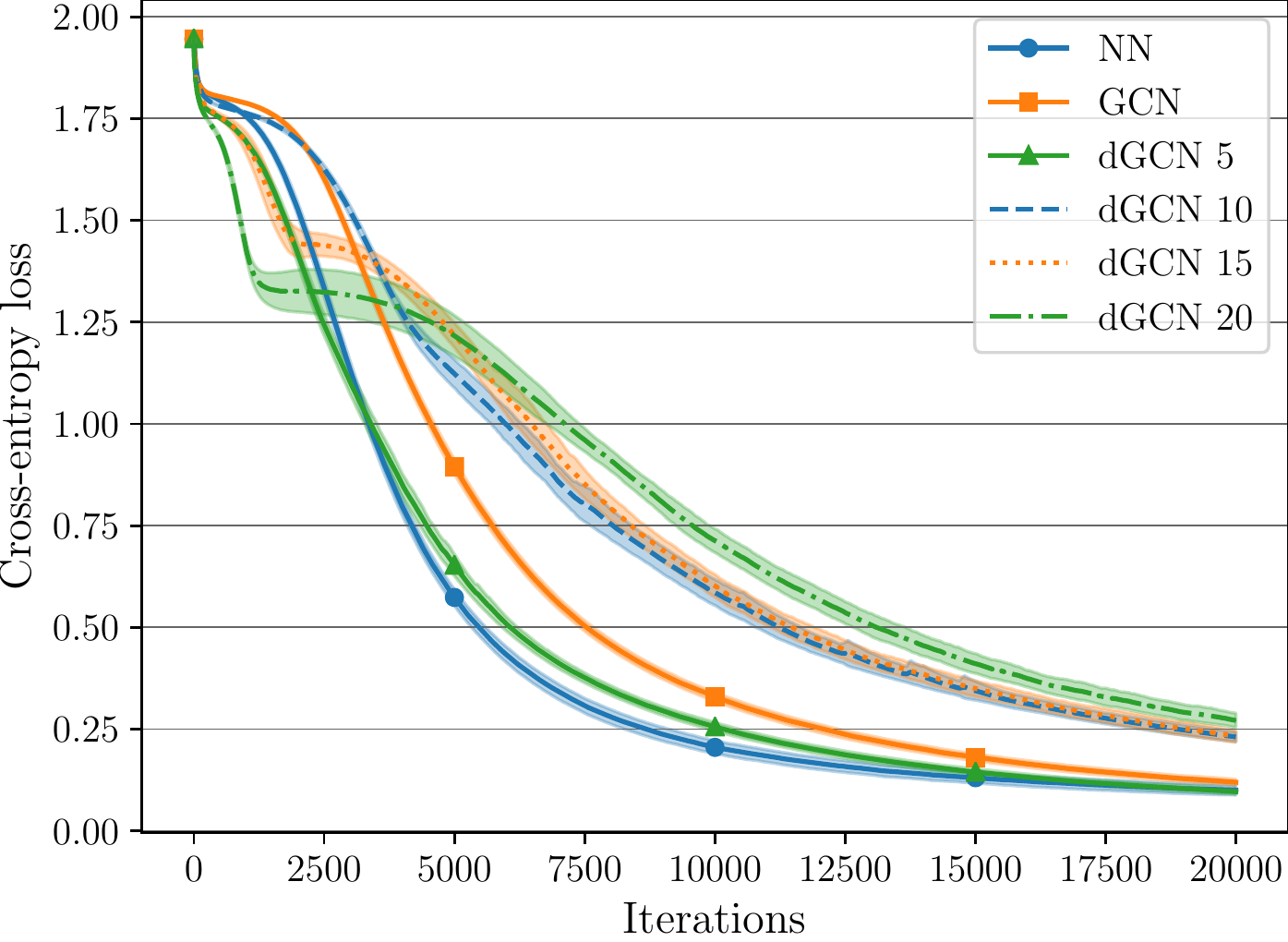}
}
\subfloat[CITESEER (loss)]{
\includegraphics[width=0.63\columnwidth,keepaspectratio]{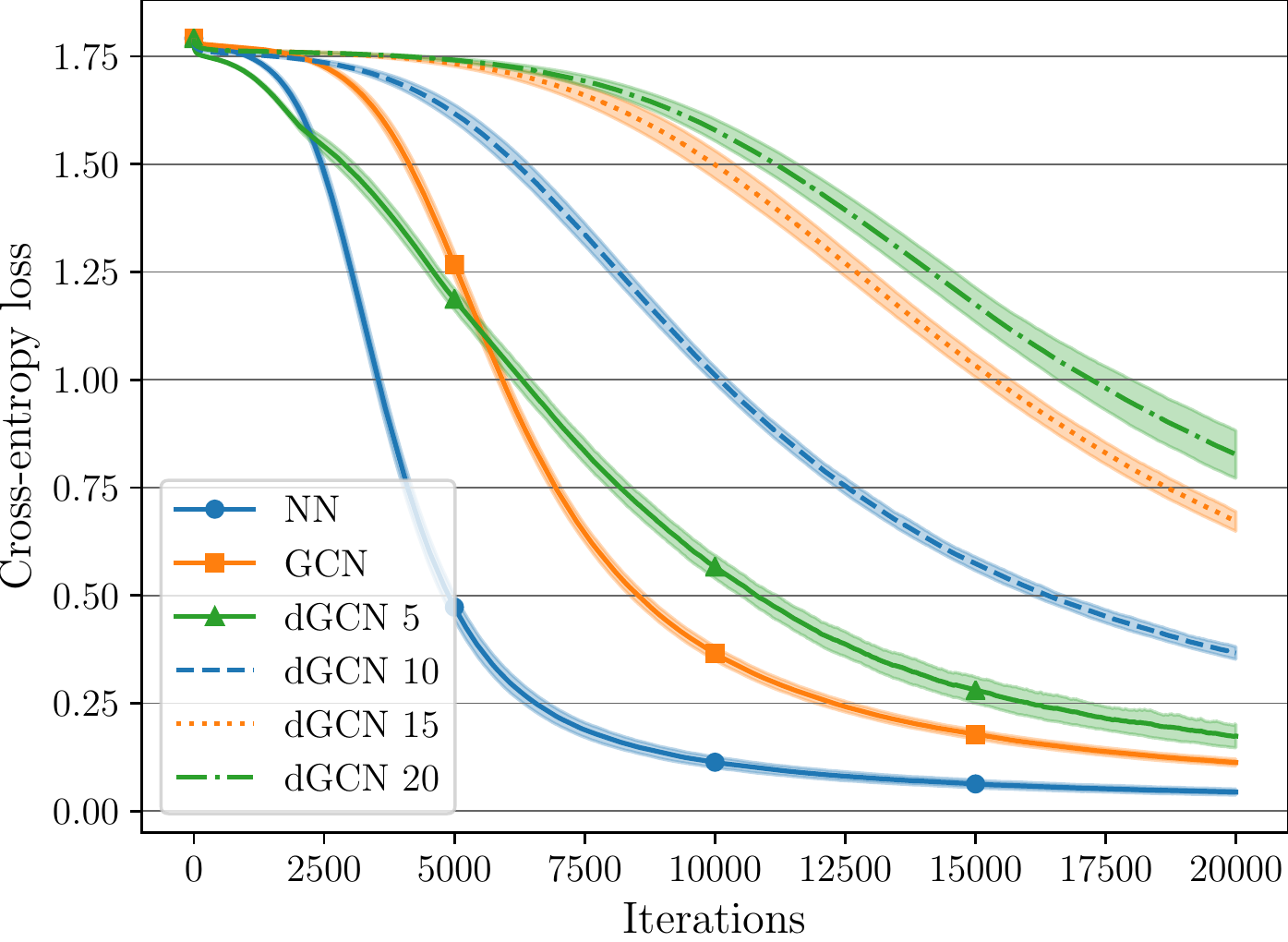}
}
\subfloat[PUBMED (loss)]{
\includegraphics[width=0.63\columnwidth,keepaspectratio]{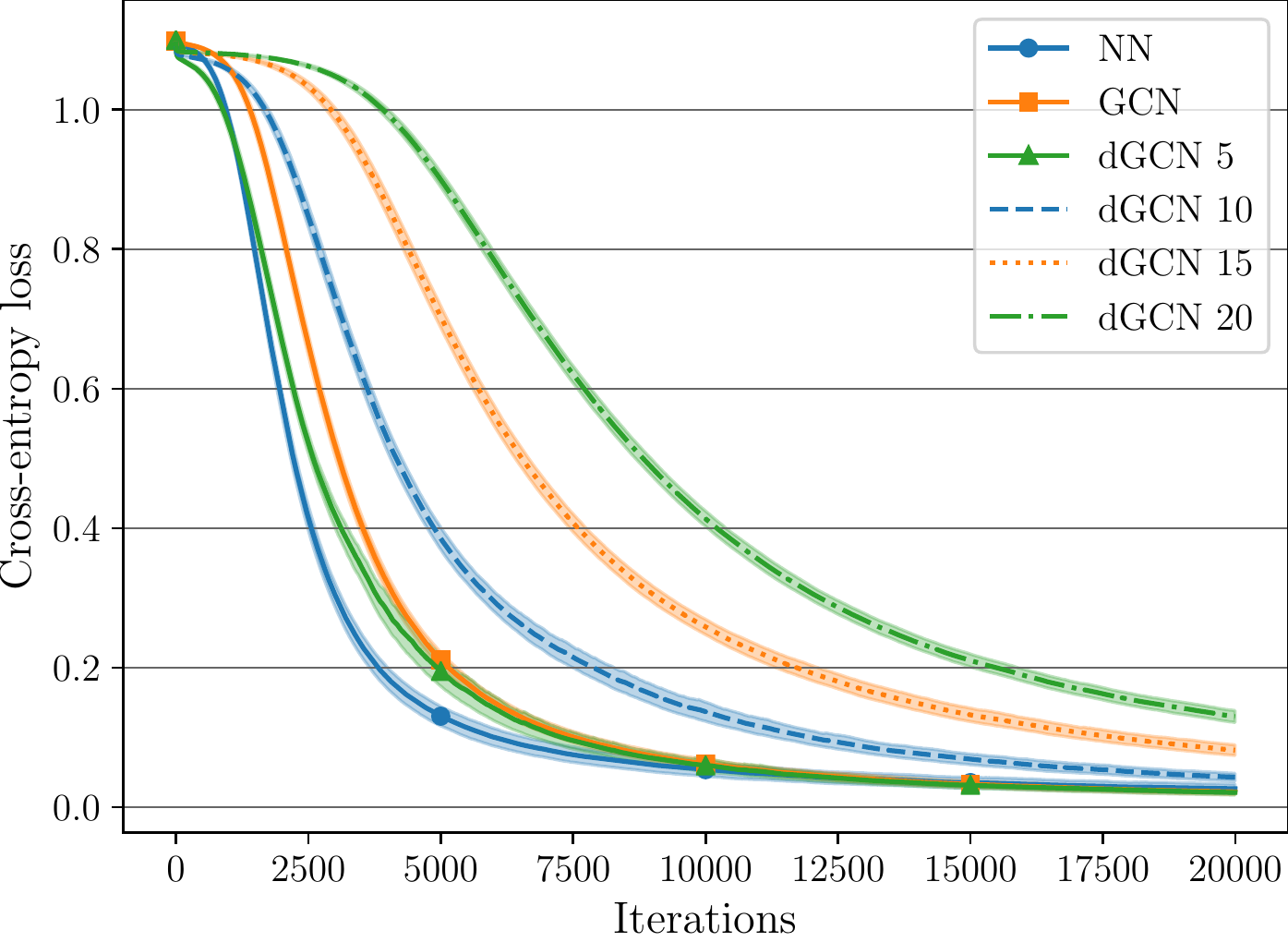}
} \\
\subfloat[CORA (accuracy)]{
\includegraphics[width=0.63\columnwidth,keepaspectratio]{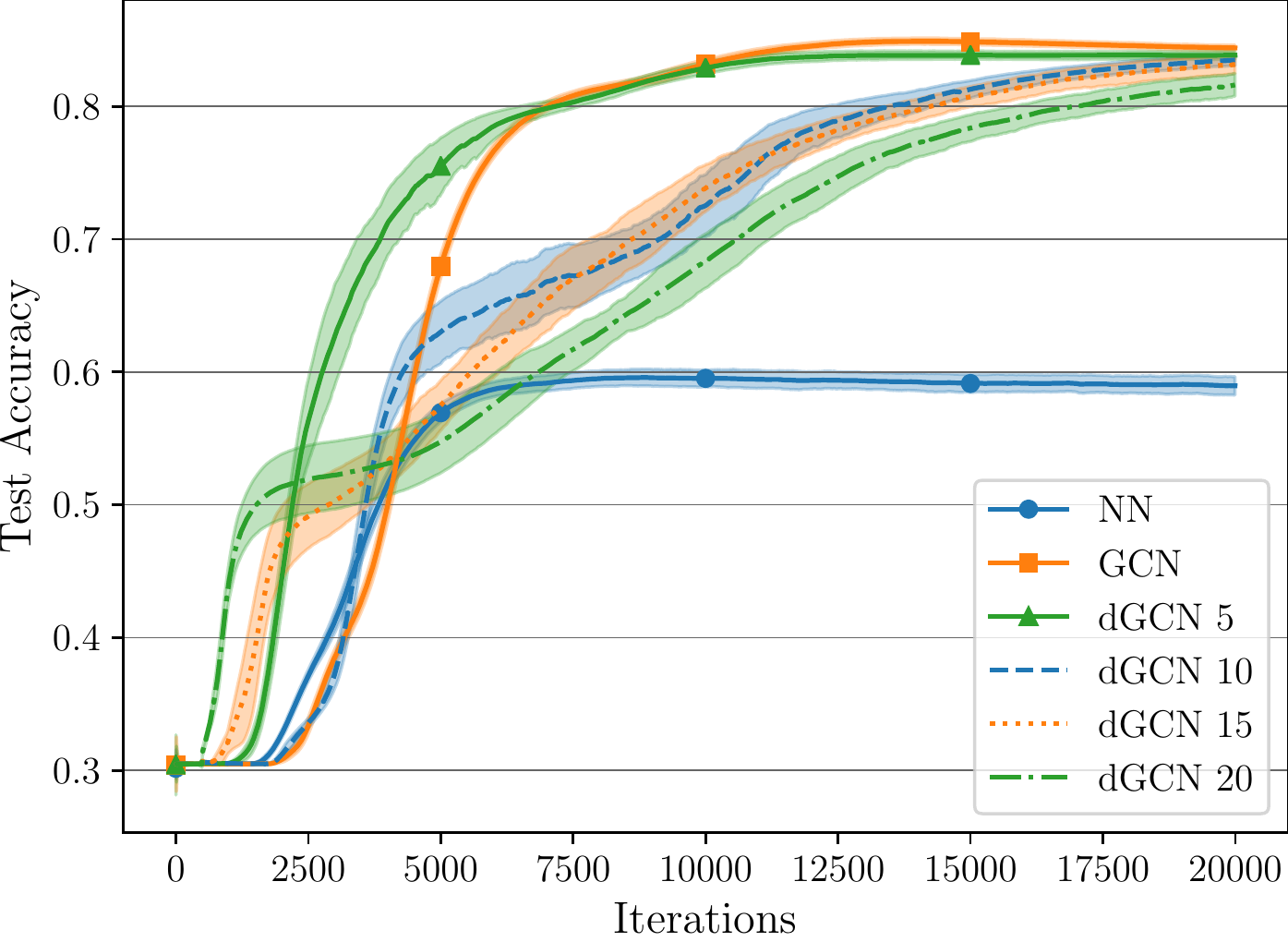}
}
\subfloat[CITESEER (accuracy)]{
\includegraphics[width=0.63\columnwidth,keepaspectratio]{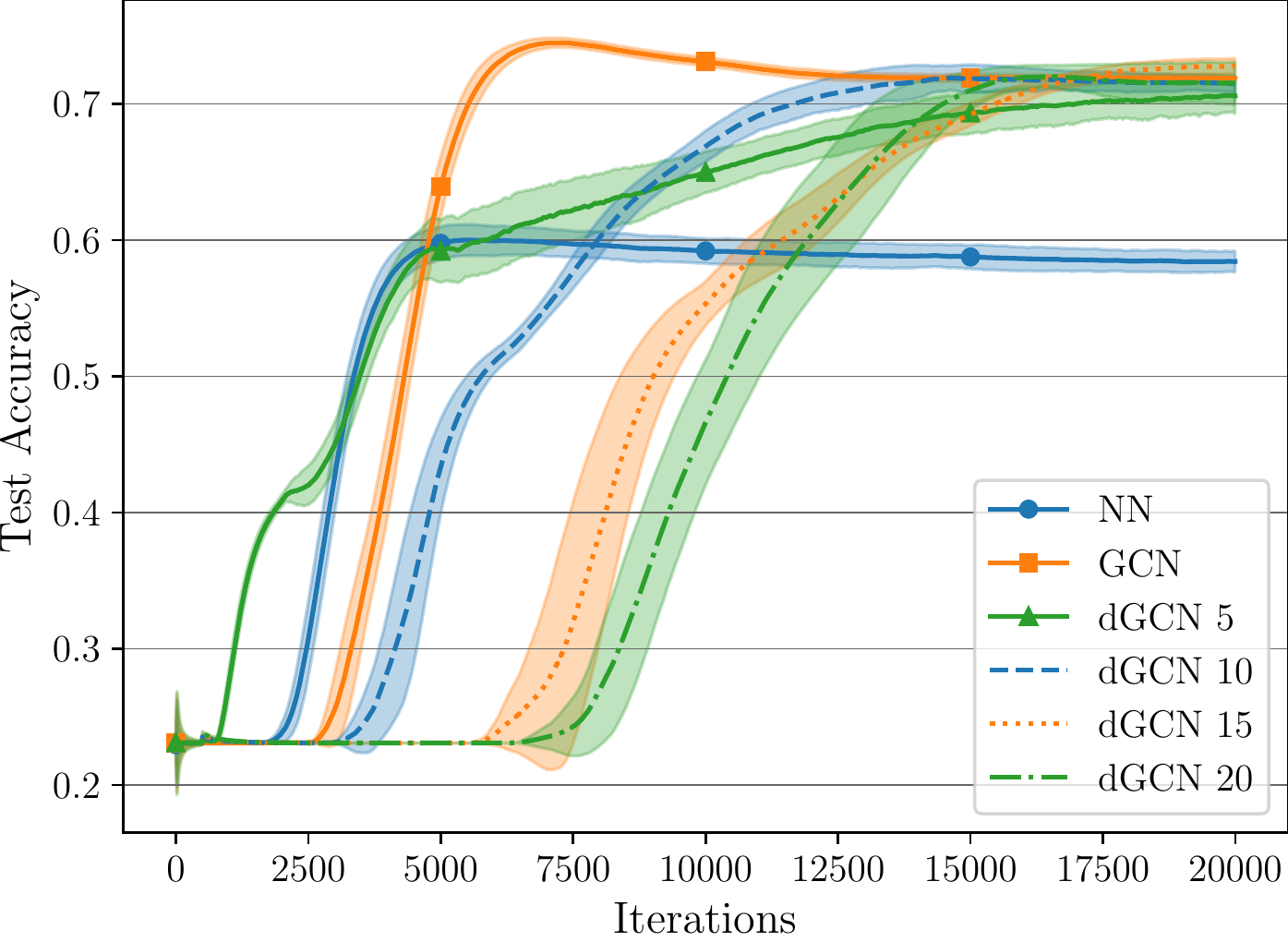}
}
\subfloat[PUBMED (accuracy)]{
\includegraphics[width=0.63\columnwidth,keepaspectratio]{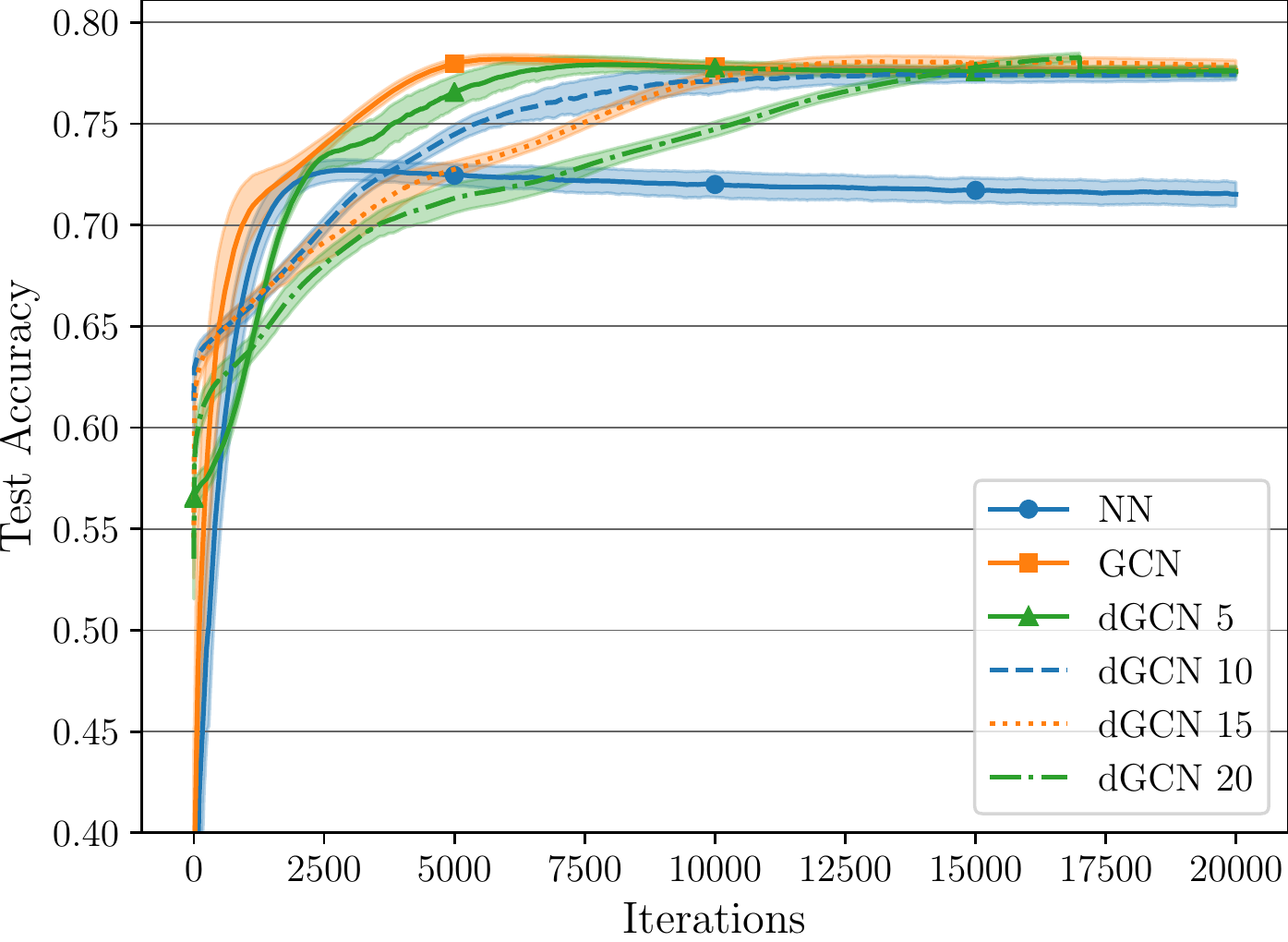}
}
\label{fig:losses}
\caption{Results for the proposed approach, compared to a standard GCN and a standard NN. First row are the training losses, second row are the test accuracies. (a)-(d) CORA; (b)-(e) CITESEER; (c)-(f) PUBMED.}
\end{figure*}

\textit{Datasets}: We apply our proposed approach to three node classification benchmarks taken from the literature on graph NNs, see \cite{kipf2017semi}. The characteristics of each dataset are briefly summarized in Table \ref{tab:datasets}, where we report the size of the underlying data graph, the number of classes for each problem, and the amount of nodes that are in the training subset. The nodes in the three datasets refer to scientific publications. Each node is endowed with a binary feature vector denoting the presence/absence of a set of words in the text. An edge in the graph represents a citation among two publications, and texts must be grouped in one of several categories.
A deeper description of the three datasets and their train/test splits can be found in \cite{yang2016revisiting}.

\textit{Architectures}: Our baseline is a standard GCN inspired by \cite{kipf2017semi}. It is built of two GC layers as in \eqref{eq:two_layer_gcn}. The hidden layer has $64$ units and ReLU activations. A dropout layer with probability $0.5$ is inserted after the input and after the hidden layer to stabilize training. As a second comparison, we consider the same GCN, but we set $\mathbf{D} = \mathbf{I}$, thus eliminating any contribution from the graph and obtaining a classical NN. All weight matrices are initialized from a Gaussian distribution with mean zero and standard deviation $10^{-3}$.  The standard GCN and NN share the same initialization, while for the distributed GCN (dGCN) we initialize a separate weight matrix for each agent.

\textit{Agents' Network}: We experiment with $5-10-15-20$ agents. Because the three benchmarks have low degree in the data graph \cite{kipf2017semi}, randomly partitioning data across the agents results with very high probability in fully-connected networks even for networks having $20$ agents. For this reason, we build a more realistic scenario by selecting $m$ seed nodes at random in the data graph, and then expanding their graph randomly in a breadth-first search. In this way, we obtain relatively well clustered assignments, such that sparsity between agents is preserved. To build the agents' connectivity graph $\mathbf{C}$, we then solve the optimization problem from Section \ref{sec:optimizing_consensus_matrix} with $\gamma=0.5$.

%\textcolor{red}{Il caso con 3 utenti è un po' borderline. Ha senso nel caso distribuito solo se non si forma una clique. Scegliendo $\gamma=0.5$, non vi vengono pienamente connesse tutte le reti? Se è così, dobbiamo aumentare il valore di $\gamma$ (vedi nuova formulazione dei problemi) e rifare le curve. In tal caso, non metterei il caso 3 agenti. Farei qualcosa tipo 10, 20, 30 agenti, bloccando il valore di $\gamma$ per far sì che le reti non vengano pienamente connesse.}.

\renewcommand{\arraystretch}{1.2}
\begin{table}[t]
\centering
\normalsize
\caption{Summary of the three benchmarks used in Section \ref{sec:experimental_evaluation}.}
\begin{tabular}{lccccc}
\textbf{Dataset} & $\lvert \mathcal{V}_D \rvert$ & $\lvert \mathcal{E}_D \rvert$ & $\lvert \mathcal{T}_D \rvert$ &  \textbf{Classes} & \textbf{Features}  \\
\hline
CITESEER & 2110 & 3668 & 120 & 6 & 3703 \\
CORA  & 2810 & 7981 & 140 & 7 & 2879 \\
PUBMED  & 19717 & 44324 & 60 & 3 & 500 \\
\end{tabular}
\label{tab:datasets}
\end{table}

\textit{Experiments}: The centralized architectures (i.e., NN and GCN) are trained using gradient descent; whereas, Algorithm 1 is used to train our distributed GCN. In all cases, we used a constant step-size rule, where the learning rate is fine-tuned manually for each dataset/algorithm to obtain the fastest convergence speed. All experiments are averaged $10$ times with respect to the initial weights, and figures illustrate both the mean and the standard deviation. All experiments are implemented in JAX \cite{jax2018github}.

\subsection{Comparisons and discussion}
In Fig. 3 we summarize the results of the simulations. Each column represents a dataset, while we show the training losses on the upper row (in log scale) and the test accuracies on the bottom row. The first trivial observation we can make is that the GCN clearly outperforms the standard NN in all three scenarios in terms of test accuracy. While this is far from being novel \cite{kipf2017semi}, it is worth underlining it here as the clearest motivation for proposing a distributed protocol for graph convolutional networks. Secondly, we can see that the dGCN asymptotically matches the performance of the centralized GCN when looking at accuracy in all three cases. In term of losses, it matches very closely the performance of GCN for CORA and PUBMED, while we see a slightly slower convergence rate for CITESEER. In Section \ref{subsec:different_update_rule} we look at how to close this gap by modifying the update rule of the proposed approach. \change{In all our experiments, consensus between the parameters belonging to different agents is reached relatively quickly, with the average absolute distance between any two set of parameters decreasing below $10^{-2}$ in approximately 200-300 iterations. The averaging step of our algorithm (line 5 in Algorithm \ref{algo:distributed_gcn_training}) is essential to guarantee stability, as we show with an ablation study later on in Section \ref{subsec:ablation_study}.}

\subsection{Results with sparse connectivity}
\label{subsec:results_with_sparse_connectivity}

Up to now, we assumed that the connectivity between agents matches the connectivity between data points, i.e., \eqref{eq:gc_layer_dist} is feasible according to the agents' topology. Next, we evaluate the performance of our method whenever this assumption is relaxed and the connectivity of the agents does not exactly match the connectivity of the underlying data graph (e.g., because of physical constraints in the available communications among agents). To this end, we replicate the experiments of the previous section, but we randomly drop up to $75\%$ of the connections between agents, while ensuring that the agents' graph remains connected overall. In addition, we experiment with ring and line topologies, to show the effect of the most minimal connectivity. After this, we renormalize both the data graph and the agents' graph and run our dGCN. First, we can evaluate the impact of removing all data graph connections that are not consistent with the new agents' connectivity in Fig. \ref{fig:drop}(a), where we show the number of edges remaining in $\mathcal{E}_D$ after the renormalization step. Note that, even in the most minimal connectivity (the line one), only up to $35\%$ of the original data connections are lost. The results in terms of training loss and test accuracies for the CORA dataset with $10$ agents are shown in Fig. \ref{fig:drop}(b)-(c). As we can notice from Fig. \ref{fig:drop}, a strong reduction of network connectivity impacts mildly the performance of the proposed strategy. Results for the other two datasets are similar and are omitted for brevity.

\begin{figure*}
    \centering
    \subfloat[Sparsity]{\includegraphics[width=0.62\columnwidth]{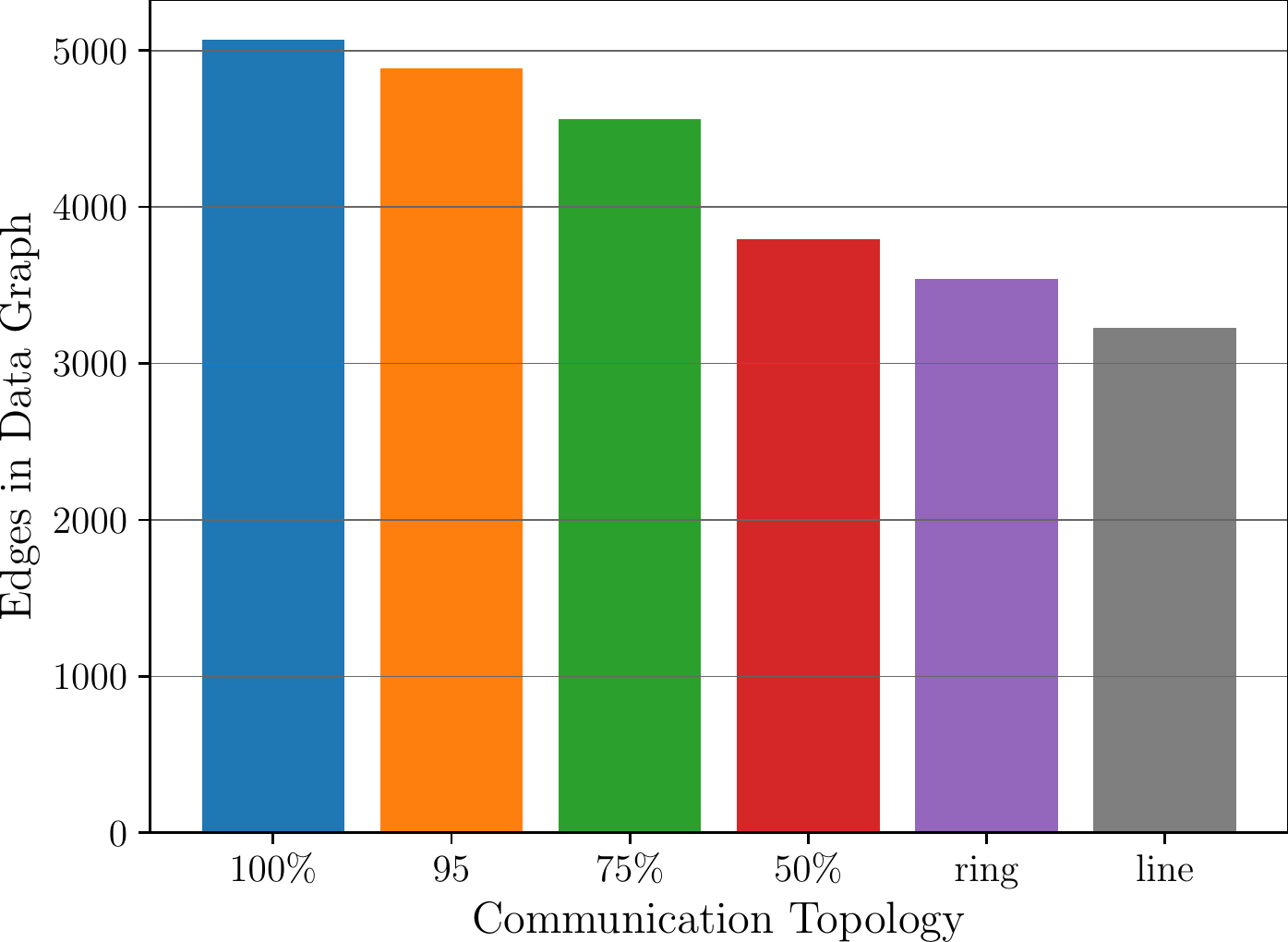}}\qquad
    \subfloat[Loss]{\includegraphics[width=0.62\columnwidth]{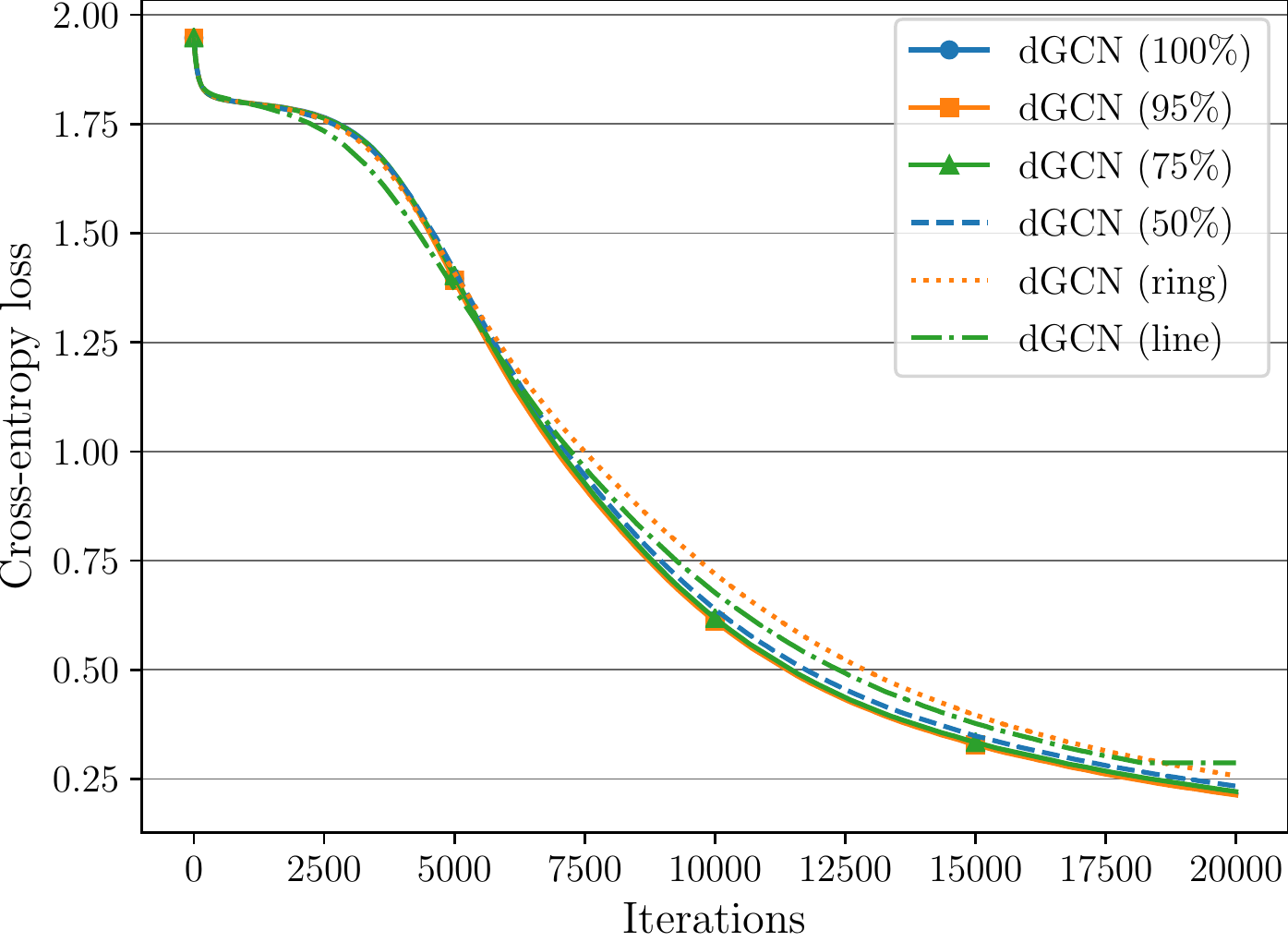}}\qquad
    \subfloat[Accuracy]{\includegraphics[width=0.62\columnwidth]{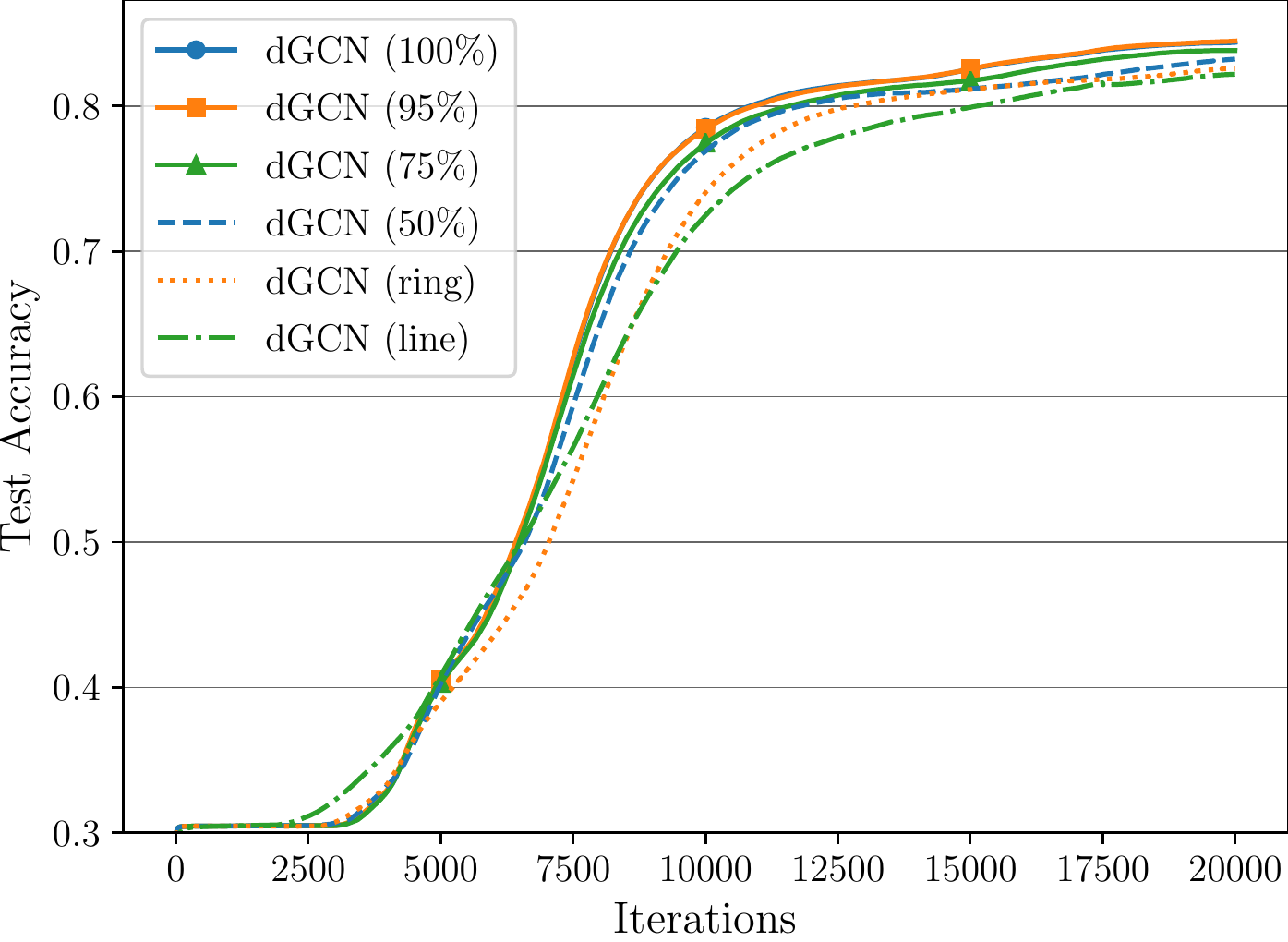}}
    \caption{\change{Results of the proposed approach when the agents' connectivity is sparse. A percentage between brackets denotes how many edges in the agents' graph have been kept, while the remaining ones have been randomly removed. (a) Number of edges remaining in $\mathcal{E}_D$ (for CORA) after removing those not consistent with a sparse agents' connectivity graph. (b) Training loss. (c) Test accuracy.}}
    \label{fig:drop}
\end{figure*}

%\begin{figure*}
%    \centering
%    \subfloat[Loss]{\includegraphics[width=0.95\columnwidth]{Images/smooth_loss_opti.pdf}}
%    \subfloat[Accuracy]{\includegraphics[width=0.95\columnwidth]{Images/smooth_acc_opti.pdf}}
%    \caption{Results of the proposed approach when replacing the gradient descent step with a momentum optimizer or an Adam-like step. (a) Training loss. (b) Test accuracy.}
%    \label{fig:adam}
%\end{figure*}

\begin{figure}
    \centering
    \includegraphics[width=0.85\columnwidth]{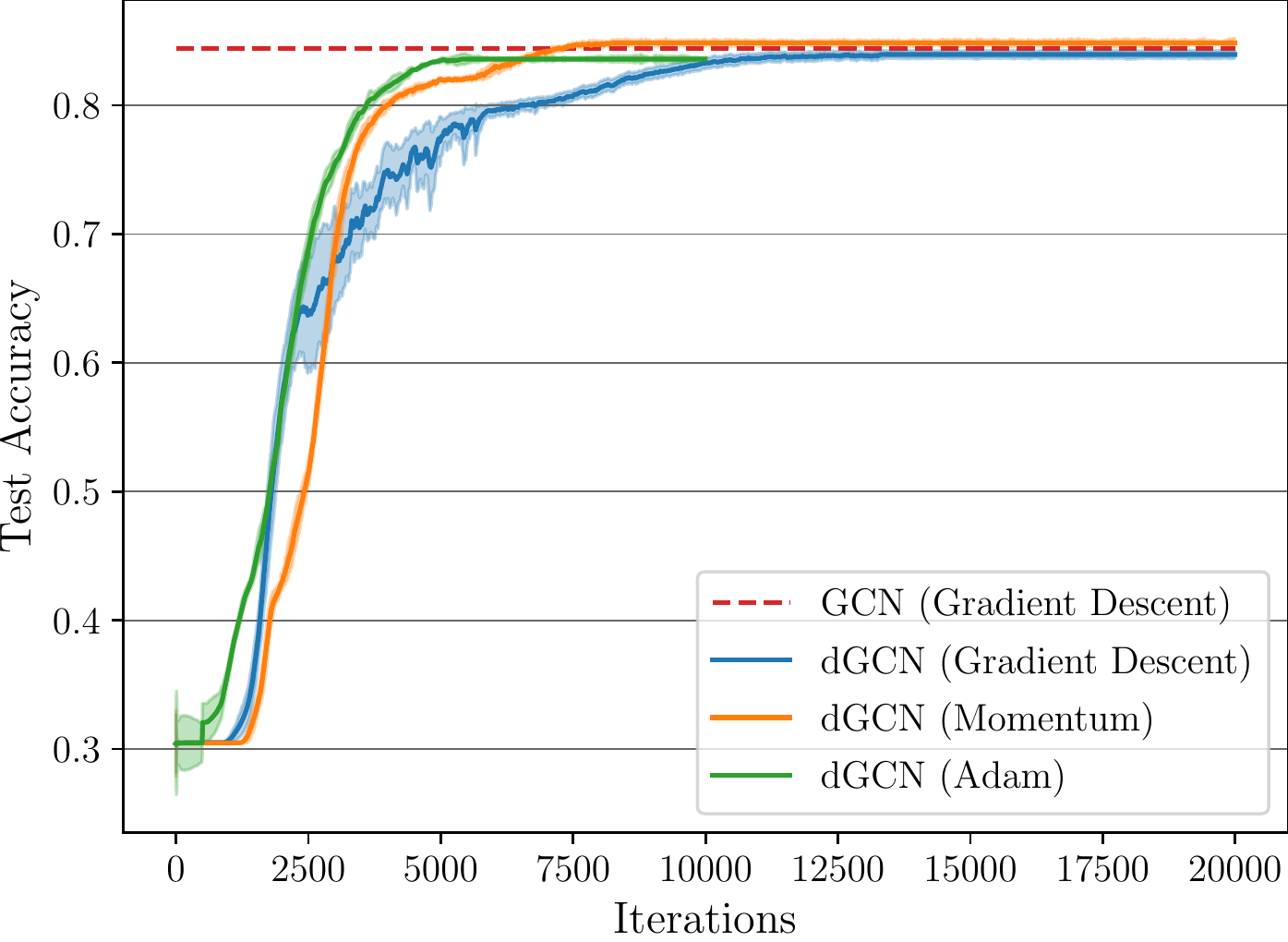}
    \caption{Results of the proposed approach when replacing the gradient descent step with a momentum optimizer or an Adam-like step (training loss).}
    \label{fig:adam}
\end{figure}

\subsection{Results with a different update rule}
\label{subsec:different_update_rule}
While we framed our method in the context of a standard gradient descent, the only requirement for convergence is that the local update rule for every agent provides a reasonably good descent direction. To this end, we replace the GD step in \eqref{eq:dist_gd} with an Adam and a momentum optimizers, where the additional hyper-parameters are always manually fine-tuned for the quickest convergence. 

In our implementation, for Adam every agent keeps track of a local set of first-order and second-order statistics of the data that are used for the update step, while for momentum we simply reuse a running average of the agents' local gradients. We consider a consensus step on the weights after updating the local statistics, although we have found no noticeable difference in interleaving the two. The results are shown in Fig. \ref{fig:adam} for a maximum number of $20000$ iterations of optimization. We see that both techniques are able to vastly accelerate the algorithm compared to a naive gradient descent. This is especially important in a distributed context, where each iteration requires exchanging data over the network.

\begin{figure}
    \centering
    \subfloat[Loss]{\includegraphics[width=0.85\columnwidth]{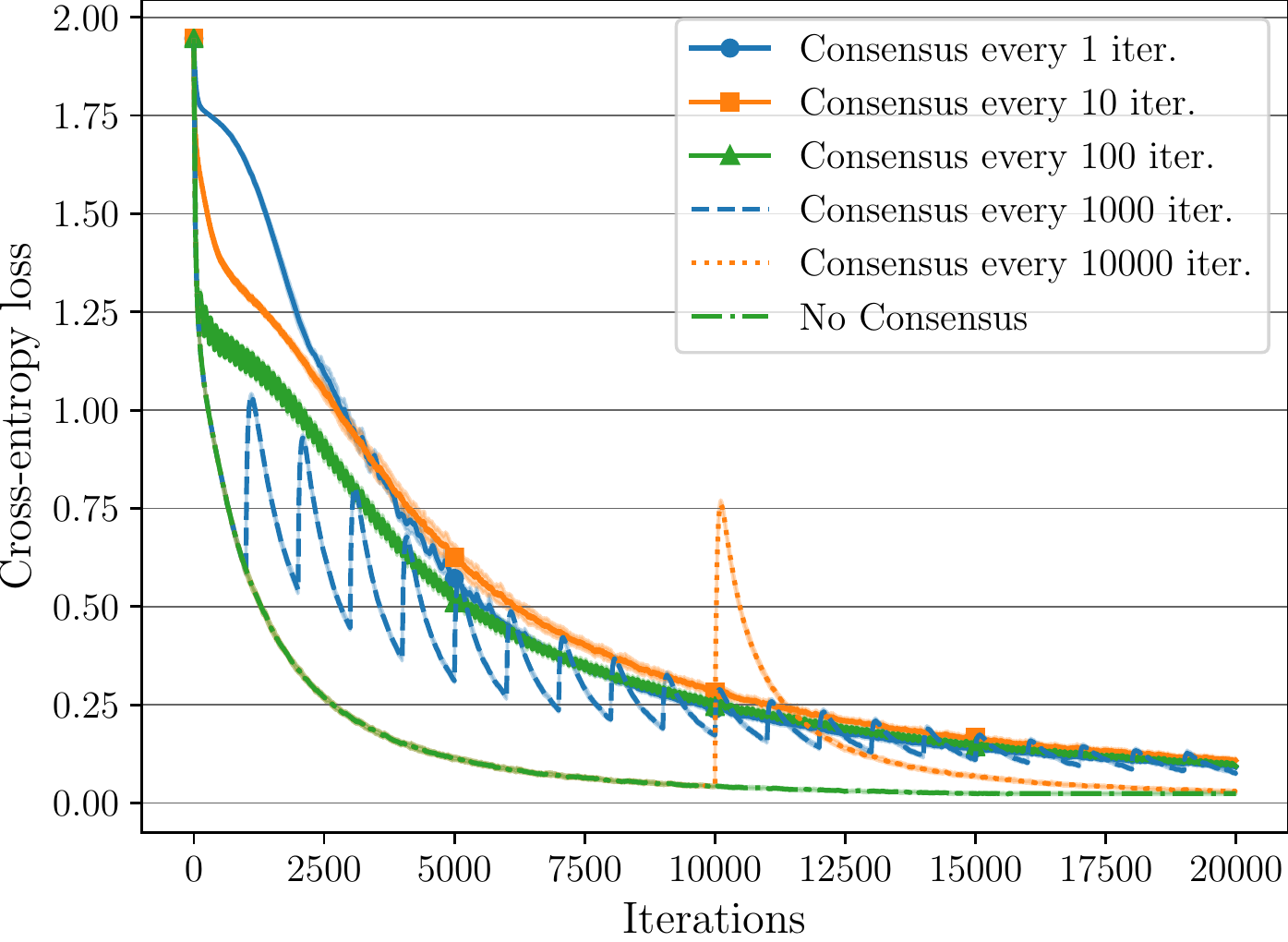}} \\
    \subfloat[Accuracy]{\includegraphics[width=0.85\columnwidth]{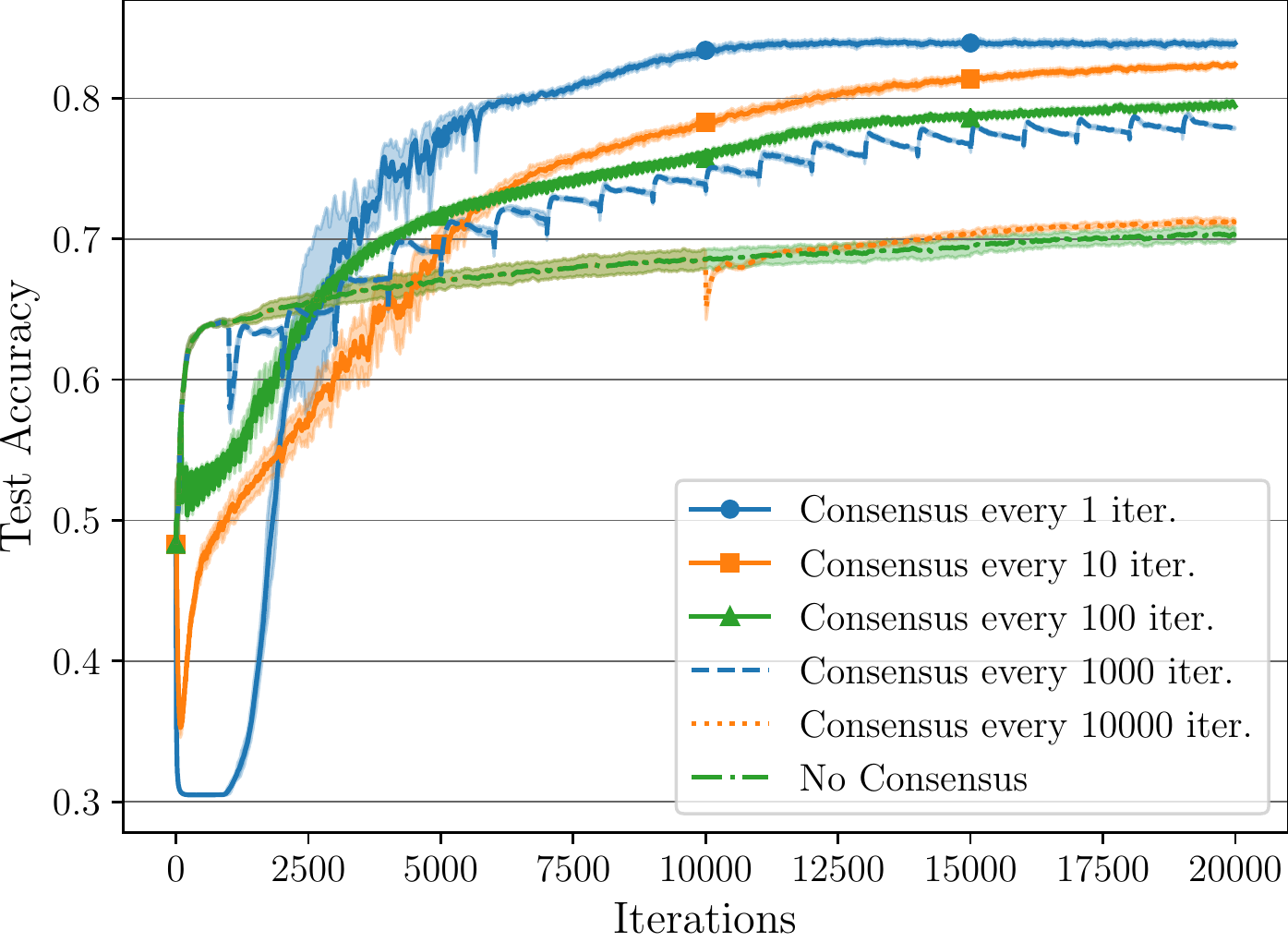}}
    \caption{Results of the proposed approach when considering consensus steps every $x$ gradient descent steps. (a) Training loss. (b) Test accuracy.}
    \label{fig:ablation_study}
\end{figure}

\subsection{Ablation study}
\label{subsec:ablation_study}

In this paragraph, we illustrate a simple ablation study on the consensus step  \eqref{eq:consensus_step} \change{(i.e., step 5 in Algorithm \ref{algo:distributed_gcn_training})}. In particular, we evaluate the results of the algorithm when multiple gradient descent steps are executed before each consensus step, and with an ablated version where no consensus is ever performed. The results are shown in Fig. \ref{fig:ablation_study}. From \ref{fig:ablation_study}, we can see that running a consensus step every iteration is required for reaching the highest accuracy.  In fact, long intervals with no consensus tend to destabilize the training phase, while removing the consensus step quickly makes all agents go into a strong overfitting phase. \change{This last behaviour can be understood as each agent minimizing the loss only on its own local portion of the training data graph, with minimal influence from the outside parts (i.e., the parts belonging to other agents). In fact, we expect the importance of the consensus constraint to grow as the ratio between the number of agents and the overall size of the data graph decreases.}

\change{
\subsection{Experiments on a traffic prediction scenario}
\label{subsec:traffic}

\renewcommand{\arraystretch}{1.2}
\begin{table}[t]
\centering
\small
\caption{Results of three distributed algorithms on the traffic prediction benchmarks from Section \ref{subsec:traffic}}
\begin{tabular}{l|c|c|c|c}
\textbf{Dataset} & Linear & GCN  & Order-2 GCN \\
\hline
PeMSD4 & 0.231 $\pm$ 0.001 & 0.148 $\pm$ 0.002 & 0.126 $\pm$ 0.001 \\
PeMSD8 & 0.266 $\pm$ 0.001 & 0.185 $\pm$ 0.002 & 0.147 $\pm$ 0.003
\end{tabular}
\label{tab:traff_pred}
\end{table}

We conclude the experimental section validating our model on two highway traffic prediction datasets
(PeMSD4 and PeMSD8) collected by the Caltrans Performance Measurement System
(PeMS) of California \cite{chen2001pems}. The measurement system collects data from more than 39000 sensors every 30 seconds. PeMSD4 aggregates the data in January and February 2018 of the San Francisco Bay Area, while PeMSD8 aggregates the data from July to September 2016 in San Bernardino. We use the same preprocessing procedure from \cite{guo2019astg}. In particular, we resample the data at 5 minutes interval, remove any sensor that is distant more than 3.5 km from any other sensor, and fill any missing value using linear interpolation. We also perform the same train/test split as in \cite{guo2019astg}. The adjacency data matrix is given by $A_{ij} = \exp\left(-d_{ij}^2/10\right)$, with $d_{ij}$ being the distance between two sensors \cite{yu2018spatio}. The matrix is then sparsified by removing all elements below 0.5 \cite{yu2018spatio}.

We assume that each sensor sends its measurements to the closest among $6$ base stations that we artificially place on a ring topology across the area (i.e., we have $m=6$ agents equispaced on a ring connectivity). For each timestep $t$, the input $\mathbf{x}_i$ for each sensor is a collection of the last $6$ recordings (i.e., the last half hour of measurements). Each sensor is then tasked with predicting the successive measurement. The loss function is computed by averaging the squared error for all timesteps in the training dataset and for all sensors. Then, beyond the distributed GCN of the previous section, we also consider a distributed linear model obtained by removing the hidden layer from the GCN definition, and a distributed order-2 GCN (see Section \ref{subsec:extension_to_order_p}) implemented with Chebyshev polynomials \cite{gama2020signals}. We average the results over $5$ different initialization of the parameters, and we report the resulting mean-squared error on the test portion of the datasets in Table \ref{tab:traff_pred}. Both distributed neural networks models significantly outperform the distributed linear one, showing the importance of having distributed training protocols for graph neural networks. The order-2 GCN further improves the results, at the cost of requiring approximately 3 times the amount of parameters exchanged over the network (as each layer is now parameterized by three different weight matrices).
}

\section{Conclusions and future work}
\label{sec:conclusions_and_future_work}
%
%The majority of works on distributed optimization of machine learning models consider classical data, such as vectors or images. 

In this paper, we propose the first algorithmic framework for distributed training of graph convolutional networks, which can be implemented in the absence of any form of central coordination. The proposed method extends previous works on distributed machine learning to the case data have some relational structure, in the form of pairwise connections between points. The resulting model is naturally distributed along three different lines: (i) during inference, (ii) during backpropagation, and (iii) during consensus phase.

Because this is one of the first works connecting these two fields of research, several extensions and contributions are possible. On one hand, our algorithm requires three communication steps instead of a single consensus, opening the possibility of compression and/or reduction protocols for the number of messages exchanged across the network. In addition, we plan to evaluate the algorithm on different graph NN variants and problems, especially in situations wherein the agents' connectivity might be time-varying. \change{An open challenge concerns situations in which our assumption of a proper matching between the data graph and the agents' graph is not satisfied. In this case, `lost' edges in the data graph should be recovered exploiting multi-hop exchanges of information among agents, which will be explored in future works.} Finally, an important open research question in graph neural networks is scaling up to larger graphs and streaming data. We believe our approach might represent an important stepping stone in this exciting research direction.

\appendices

\section{A Lemma on convergence of scalar sequences}
We introduce a lemma dealing with convergence of some sequences, which will be used in the following arguments.\smallskip

\begin{lemma}\label{Lemma_sequences}
Let $0<\zeta<1$, and  let $\{\alpha^t\}$ and  $\{\nu^t\}$ be two positive scalar sequences. Then, the following hold:
\begin{description}
  \item[(a)] If $\displaystyle\lim_{t\rightarrow\infty}\alpha^t=0$, then\vspace{-0.2cm}
   \begin{equation}\label{lemma3a}
       \displaystyle \lim_{t\rightarrow\infty}\,\sum_{n=1}^{t}\zeta^{t-n}\alpha^n=0.
    \end{equation}
  \item[(b)] If $\sum_{t=1}^\infty (\alpha^t)^2<\infty$ and $\sum_{t=1}^\infty (\nu^t)^2<\infty$, then
   \begin{align}\label{lemma3b}
%     & \hbox{\rm (b.1):}\;\;\displaystyle\lim_{n\rightarrow\infty}\, \sum_{k=1}^{n}\sum_{l=1}^{k}\lambda^{k-l}\beta[l]^2<\infty, \label{lemma1b1}\\
     \displaystyle\lim_{t\rightarrow\infty}\, \sum_{n=1}^{t}\sum_{l=1}^{n}\zeta^{n-l}\alpha^n\nu^l<\infty.
   \end{align}
\end{description}
\end{lemma}

\section{Proof of Theorem 2}

\textbf{Point (a):} Let $\mathbf{\Pi}_{\mathcal{S}^\perp}= \mathbf{I}-\mathbf{\Pi}_{\mathcal{S}}$ be the projector onto the subspace orthogonal to $\mathcal{S}$. We will now study the temporal evolution of
\begin{equation}\label{x_ort}
\overline{\mathbf{w}}_{\mathcal{S}^\perp}^t= \mathbf{\Pi}_{\mathcal{S}^\perp}\overline{\mathbf{w}}^t=
\overline{\mathbf{w}}^t-\mathbf{\Pi}_{\mathcal{S}}\overline{\mathbf{w}}^t= \overline{\mathbf{w}}^t- \overline{\mathbf{w}}_{\mathcal{S}}^t
\end{equation}
i.e., the component of $\overline{\mathbf{w}}^t$ that lies in the subspace orthogonal to $\mathcal{S}$. To this aim, multiplying (\ref{Dist_gradient_proj}) from the left side by $\mathbf{\Pi}_{\mathcal{S}^\perp}$, and letting $\mathbf{c}^t=\mathbf{\Pi}_{\mathcal{S}^\perp} (\mathbf{C}\otimes \mathbf{I}_p) \nabla L(\overline{\mathbf{w}}^t)$, we obtain:
\begin{align}\label{orthogonal_recursion}
\mathbf{\Pi}_{\mathcal{S}^\perp}\overline{\mathbf{w}}^{t+1}\,=\,\mathbf{\Pi}_{\mathcal{S}^\perp}(\mathbf{C}\otimes \mathbf{I}_p)\overline{\mathbf{w}}^t-\eta^t\mathbf{c}^t.
\end{align}
Now, since $\mathbf{\Pi}_{\mathcal{S}^\perp}(\mathbf{C}\otimes \mathbf{I}_p)=\mathbf{\Pi}_{\mathcal{S}^\perp}(\mathbf{C}\otimes \mathbf{I}_p)\mathbf{\Pi}_{\mathcal{S}^\perp}$ if (A1) holds, eq. (\ref{orthogonal_recursion}) becomes:

\begin{align}\label{orthogonal_recursion2}
\overline{\mathbf{w}}_{\mathcal{S}^\perp}^{t+1}=&\,=\,
\mathbf{\Pi}_{\mathcal{S}^\perp}(\mathbf{C}\otimes \mathbf{I}_p)\,\overline{\mathbf{w}}_{\mathcal{S}^\perp}^{t}-\eta^t\mathbf{c}^t\nonumber\\
&\,=\, \bigg(\Big(\mathbf{C}- \frac{1}{m}\mathbf{1}_m\mathbf{1}_m^T\Big)\otimes \mathbf{I}_p\bigg)\overline{\mathbf{w}}_{\mathcal{S}^\perp}^{t}-\eta^t\mathbf{c}^t.
\end{align}
Iterating recursion (\ref{orthogonal_recursion2}), we have:
\begin{align}\label{orthogonal_recursion3}
\overline{\mathbf{w}}_{\mathcal{S}^\perp}^{t}\,=\,&\,\bigg(\Big(\mathbf{C}- \frac{1}{m}\mathbf{1}_m\mathbf{1}_m^T\Big)\otimes \mathbf{I}_p\bigg)^t\overline{\mathbf{w}}_{\mathcal{S}^\perp}^{0} \nonumber\\
&-\sum_{l=1}^t\bigg(\Big(\mathbf{C}- \frac{1}{m}\mathbf{1}_m\mathbf{1}_m^T\Big)\otimes \mathbf{I}_p\bigg)^{t-l}\eta^{l-1}\mathbf{c}^{l-1}.
\end{align}
Now, taking the norm of (\ref{orthogonal_recursion3}) and using (B2), (A2), we obtain:
\begin{align}\label{orthogonal_recursion4}
\|\overline{\mathbf{w}}_{\mathcal{S}^\perp}^{t}\|\,\leq\,&\,(1-\gamma)^t\|\overline{\mathbf{w}}_{\mathcal{S}^\perp}^{0}\|+G\sum_{l=1}^t(1-\gamma)^{t-l}\eta^{l-1}.
\end{align}
Then, if (C1) holds, taking the limit of (\ref{orthogonal_recursion4}), since $(1-\gamma)^t\rightarrow 0$ as $t\rightarrow\infty$
and $\displaystyle\lim_{t\rightarrow\infty}\sum_{l=1}^t(1-\gamma)^{t-l}=\frac{1}{\gamma}$, we have:
\begin{align}\label{orthogonal_limit1}
\lim_{t\rightarrow\infty}\|\overline{\mathbf{w}}_{\mathcal{S}^\perp}^{t}\|\,\leq\,&\,\frac{G\eta}{\gamma}=O(\eta),
\end{align}
which proves (\ref{sub_lim1}) [cf. (\ref{x_ort})]. On the other side, if (C2) holds, invoking Lemma 3(a) [cf. (\ref{lemma3a})], from (\ref{orthogonal_recursion4}) we conclude that:
\begin{align}\label{orthogonal_limit2}
\lim_{t\rightarrow\infty}\|\overline{\mathbf{w}}_{\mathcal{S}^\perp}^{t}\|\,=\, 0,
\end{align}
thus proving also (\ref{sub_lim2}), and completing the proof of point (a).

\textbf{Point (b):} We now prove the convergent behavior of $\overline{\mathbf{w}}^t_{\mathcal{S}}$, i.e., the component of $\overline{\mathbf{w}}^t$ that lies in $\mathcal{S}$. Thus, multiplying (\ref{Dist_gradient_proj}) from the left side by $\mathbf{\Pi}_{\mathcal{S}}$, and using (A1), we obtain:
\begin{align}\label{x_bar_evolution}
\overline{\mathbf{w}}^{t+1}_{\mathcal{S}}=\overline{\mathbf{w}}^t_{\mathcal{S}}-\eta^t\mathbf{\Pi}_{\mathcal{S}}\nabla  L(\overline{\mathbf{w}}^t).
\end{align}
Now, under (A1), invoking the descent lemma on $L$ and using (\ref{x_bar_evolution}), we obtain:
\begin{align}\label{descent_function}
&L(\overline{\mathbf{w}}^{t+1}_{\mathcal{S}})\,\leq\,L(\overline{\mathbf{w}}^{t}_{\mathcal{S}})-\eta^t\nabla L(\overline{\mathbf{w}}_{\mathcal{S}}^t)^T\mathbf{\Pi}_{\mathcal{S}}\nabla L(\overline{\mathbf{w}}^t)\nonumber\\
&\qquad\qquad\quad+\frac{c_L}{2}(\eta^t)^2\|\mathbf{\Pi}_{\mathcal{S}}\nabla L(\overline{\mathbf{w}}^t)\|^2.
\end{align}
Now, summing and subtracting the vector $\nabla  L(\overline{\mathbf{w}}_{\mathcal{S}}^t)$ properly in the second term on the RHS of (\ref{descent_function}), we obtain:
\begin{align}\label{descent_function2}
L(\overline{\mathbf{w}}^{t+1}_{\mathcal{S}})\,&\stackrel{(a)}{\leq}\, L(\overline{\mathbf{w}}^{t}_{\mathcal{S}})-\eta^t \|\nabla L(\overline{\mathbf{w}}_{\mathcal{S}}^t)\|^2_{\mathbf{\Pi}_{\mathcal{S}}}+\frac{c_L}{2}G^2(\eta^t)^2  \nonumber\\
&\;\;\;-\eta^t  \nabla L(\overline{\mathbf{w}}_{\mathcal{S}}^t)^T\mathbf{\Pi}_{\mathcal{S}}\big(\nabla L(\overline{\mathbf{w}}^t)-\nabla L(\overline{\mathbf{w}}_{\mathcal{S}}^t)\big)\nonumber\\
&\hspace{-1cm}\stackrel{(b)}{\leq}\,L(\overline{\mathbf{w}}^{t}_{\mathcal{S}})-\eta^t g^t+\frac{c_L}{2}G^2(\eta^t)^2+\eta^t c_L G\|\overline{\mathbf{w}}^t-\overline{\mathbf{w}}_{\mathcal{S}}^t\|\nonumber\\
&\hspace{-1cm}\stackrel{(c)}{\leq}\, L(\overline{\mathbf{w}}^{t}_{\mathcal{S}})-\eta^t g^t+\frac{c_L}{2}G^2(\eta^t)^2+ c_L G\|\overline{\mathbf{w}}_{\mathcal{S}^\perp}^{0}\|(1-\gamma)^t \eta^t\nonumber\\
&\hspace{-.5cm} + c_L G^2 \eta^t \sum_{l=1}^t(1-\gamma)^{t-l}\eta^{l-1},
\end{align}
where in (a) we used (B2); (b) comes from (B1), (B2), and (\ref{g_nonconvex}); and in (c) we used (\ref{orthogonal_recursion4}) [cf. (\ref{x_ort})]. Now, applying recursively (\ref{descent_function2}), since under (B3) $L$ is bounded from below (w.l.o.g., we consider $L(\overline{\mathbf{w}}^{t+1}_{\mathcal{S}})\geq 0$ for all $t$), we obtain:
\begin{align}\label{descent_function3}
&\sum_{n=0}^t \eta^n g^n \,\leq\;\;  L(\overline{\mathbf{w}}^{0}_{\mathcal{S}})+ c_LG \|\overline{\mathbf{w}}_{\mathcal{S}^\perp}^{0}\|\sum_{n=0}^t(1-\gamma)^n \eta^n\nonumber\\
&\hspace{.5cm}+\frac{c_L}{2}G^2\sum_{n=0}^t (\eta^n)^2+ c_L G^2 \sum_{n=0}^t\sum_{l=1}^n(1-\gamma)^{n-l}\eta^n\eta^{l-1}.
\end{align}
Using 
%\begin{align}
$\displaystyle\sum_{n=0}^t \eta^n g^n\geq \left(\sum_{n=0}^t \eta^n\right) g_{best}^t$
%\end{align}
in (\ref{descent_function3}) [cf. (\ref{g_best})], we get:
\begin{align}\label{descent_function4}
&\hspace{-.2cm}g_{best}^t \leq \bigg( L(\overline{\mathbf{w}}^{0}_{\mathcal{S}})+ c_LG \|\overline{\mathbf{w}}_{\mathcal{S}^\perp}^{0}\|\sum_{n=0}^t(1-\gamma)^n \eta^n  +\frac{c_L}{2}G^2\sum_{n=0}^t (\eta^n)^2 \nonumber\\
&\;  
+ c_L G^2 \sum_{n=0}^t\sum_{l=1}^n(1-\gamma)^{n-l}\eta^n\eta^{l-1}
\bigg)\frac{1}{\left(\displaystyle
\sum_{n=0}^t \eta^n
\right)}. 
\end{align}
Then, if (C1) holds, exploiting $\sum_{n=0}^t(1-\gamma)^n\leq 1/\gamma$ for all $t$, from (\ref{descent_function4}) we obtain:
\begin{align}\label{descent_function5}
&g_{best}^t \leq L(\overline{\mathbf{w}}^{0}_{\mathcal{S}})+\frac{c_LG\|\overline{\mathbf{w}}_{\mathcal{S}^\perp}^{0}\|\eta}{\gamma}+\frac{c_L}{2}(t+1)G^2\eta^2\nonumber\\
&\qquad\qquad\quad+\frac{c_LG^2\eta^2}{\gamma}(t+1)\bigg)\frac{1}{(t+1)\eta}.
\end{align}
Thus, from (\ref{descent_function5}), it is easy to see that
\begin{align}\label{descent_function6}
&\lim_{t\rightarrow\infty}\;g_{best}^t\leq \frac{c_LG^2(2+\gamma)}{2\gamma}\eta\,=\, O(\eta),
\end{align}
with convergence rate $O\left(\frac{1}{t+1}\right)$, thus proving (\ref{conv_nonconvex1}).

Finally, using (C2) in (\ref{descent_function4}), and exploiting (\ref{lemma3a}) and (\ref{lemma3b}), we directly obtain: \begin{align}\label{lim_inf}
\lim_{t\rightarrow\infty}\;g_{best}^t=\displaystyle\lim_{t\rightarrow\infty}\inf_t\;\|\nabla L(\overline{\mathbf{w}}^t_{\mathcal{S}})\|_{\mathbf{\Pi}_{\mathcal{S}}} =0.
\end{align}
Using similar arguments as in \cite[p.1887]{Dan-Facch-Kung-Scut} (omitted due to lack of space), we can also prove:
\begin{align}\label{lim_sup}
\displaystyle\lim_{t\rightarrow\infty}\sup_t\;\|\nabla L(\overline{\mathbf{w}}^t_{\mathcal{S}})\|_{\mathbf{\Pi}_{\mathcal{S}}} =0.
\end{align}
In conclusion, from (\ref{lim_inf}) and (\ref{lim_sup}), we must have:
\begin{equation}\label{conv_stat_point}
\lim_{t\rightarrow\infty} \|\nabla L(\overline{\mathbf{w}}^t_{\mathcal{S}})\|_{\mathbf{\Pi}_{\mathcal{S}}} =0,
\end{equation}
i.e., the sequence $\{\overline{\mathbf{w}}^t_{\mathcal{S}}\}_t$ converges to a stationary point of (\ref{eq:opt_dist}) [cf. (\ref{stat_point})]. This concludes the proof of Theorem 2.

\bibliographystyle{IEEEtran}
\balance
\bibliography{references}

\begin{thebibliography}{10}
\providecommand{\url}[1]{#1}
\csname url@samestyle\endcsname
\providecommand{\newblock}{\relax}
\providecommand{\bibinfo}[2]{#2}
\providecommand{\BIBentrySTDinterwordspacing}{\spaceskip=0pt\relax}
\providecommand{\BIBentryALTinterwordstretchfactor}{4}
\providecommand{\BIBentryALTinterwordspacing}{\spaceskip=\fontdimen2\font plus
\BIBentryALTinterwordstretchfactor\fontdimen3\font minus
  \fontdimen4\font\relax}
\providecommand{\BIBforeignlanguage}[2]{{%
\expandafter\ifx\csname l@#1\endcsname\relax
\typeout{** WARNING: IEEEtran.bst: No hyphenation pattern has been}%
\typeout{** loaded for the language `#1'. Using the pattern for}%
\typeout{** the default language instead.}%
\else
\language=\csname l@#1\endcsname
\fi
#2}}
\providecommand{\BIBdecl}{\relax}
\BIBdecl

\bibitem{newman2002random}
M.~E. Newman, D.~J. Watts, and S.~H. Strogatz, ``Random graph models of social
  networks,'' \emph{Proceedings of the National Academy of Sciences}, vol.~99,
  no. suppl 1, pp. 2566--2572, 2002.

\bibitem{jombart2011reconstructing}
T.~Jombart, R.~Eggo, P.~Dodd, and F.~Balloux, ``Reconstructing disease
  outbreaks from genetic data: a graph approach,'' \emph{Heredity}, vol. 106,
  no.~2, pp. 383--390, 2011.

\bibitem{berg2017graph}
R.~v.~d. Berg, T.~N. Kipf, and M.~Welling, ``Graph convolutional matrix
  completion,'' \emph{arXiv preprint arXiv:1706.02263}, 2017.

\bibitem{allamanis2017learning}
M.~Allamanis, M.~Brockschmidt, and M.~Khademi, ``Learning to represent programs
  with graphs,'' in \emph{Proc. 6th International Conference on Learning
  Representations (ICLR)}, 2018.

\bibitem{romero2017kernel}
D.~Romero, M.~Ma, and G.~B. Giannakis, ``Kernel-based reconstruction of graph
  signals,'' \emph{IEEE Transactions on Signal Processing}, vol.~65, no.~3, pp.
  764--778, 2017.

\bibitem{bastings2017graph}
J.~Bastings, I.~Titov, W.~Aziz, D.~Marcheggiani, and K.~Sima'an, ``Graph
  convolutional encoders for syntax-aware neural machine translation,'' in
  \emph{Proc. 2017 Conference on Empirical Methods in Natural Language
  Processing (EMNLP)}, 2017.

\bibitem{lee2010discovering}
A.~B. Lee, D.~Luca, L.~Klei, B.~Devlin, and K.~Roeder, ``Discovering genetic
  ancestry using spectral graph theory,'' \emph{Genetic Epidemiology}, vol.~34,
  no.~1, pp. 51--59, 2010.

\bibitem{bruna2014spectral}
J.~Bruna, W.~Zaremba, A.~Szlam, and Y.~Lecun, ``Spectral networks and locally
  connected networks on graphs,'' in \emph{Proc. International Conference on
  Learning Representations (ICLR)}, 2014.

\bibitem{kipf2017semi}
T.~N. Kipf and M.~Welling, ``Semi-supervised classification with graph
  convolutional networks,'' \emph{Proc. International Conference on Learning
  Representations (ICLR)}, 2017.

\bibitem{bronstein2017geometric}
M.~M. Bronstein, J.~Bruna, Y.~LeCun, A.~Szlam, and P.~Vandergheynst,
  ``Geometric deep learning: going beyond {Euclidean} data,'' \emph{IEEE Signal
  Process. Magazine}, vol.~34, no.~4, pp. 18--42, 2017.

\bibitem{kang2012fast}
U.~Kang, H.~Tong, and J.~Sun, ``Fast random walk graph kernel,'' in \emph{Proc.
  2012 SIAM International Conference on Data Mining (SDM)}.\hskip 1em plus
  0.5em minus 0.4em\relax SIAM, 2012, pp. 828--838.

\bibitem{kondor2016multiscale}
R.~Kondor and H.~Pan, ``The multiscale laplacian graph kernel,'' in
  \emph{Advances in Neural Information Processing Systems}, 2016, pp.
  2990--2998.

\bibitem{belkin2006manifold}
M.~Belkin, P.~Niyogi, and V.~Sindhwani, ``Manifold regularization: A geometric
  framework for learning from labeled and unlabeled examples,'' \emph{Journal
  of Machine Learning Research}, vol.~7, no. Nov, pp. 2399--2434, 2006.

\bibitem{cai2010graph}
D.~Cai, X.~He, J.~Han, and T.~S. Huang, ``Graph regularized nonnegative matrix
  factorization for data representation,'' \emph{IEEE Transactions on Pattern
  Analysis and Machine Intelligence}, vol.~33, no.~8, pp. 1548--1560, 2010.

\bibitem{boyd2011distributed}
S.~Boyd, N.~Parikh, E.~Chu, B.~Peleato, and J.~Eckstein, ``Distributed
  optimization and statistical learning via the alternating direction method of
  multipliers,'' \emph{Foundations and Trends{\textregistered} in Machine
  Learning}, vol.~3, no.~1, pp. 1--122, 2011.

\bibitem{ouyang2013stochastic}
H.~Ouyang, N.~He, L.~Tran, and A.~Gray, ``Stochastic alternating direction
  method of multipliers,'' in \emph{Proc. 30th International Conference on
  Machine Learning (ICML)}, 2013, pp. 80--88.

\bibitem{cattivelli2008diffusion}
F.~S. Cattivelli, C.~G. Lopes, and A.~H. Sayed, ``Diffusion recursive
  least-squares for distributed estimation over adaptive networks,'' \emph{IEEE
  Transactions on Signal Processing}, vol.~56, no.~5, pp. 1865--1877, 2008.

\bibitem{nedic2009distributed}
A.~Nedic and A.~Ozdaglar, ``Distributed subgradient methods for multi-agent
  optimization,'' \emph{IEEE Transactions on Automatic Control}, vol.~54,
  no.~1, p.~48, 2009.

\bibitem{bianchi2012convergence}
P.~Bianchi and J.~Jakubowicz, ``Convergence of a multi-agent projected
  stochastic gradient algorithm for non-convex optimization,'' \emph{IEEE
  Transactions on Automatic Control}, vol.~58, no.~2, pp. 391--405, 2012.

\bibitem{di2016next}
P.~Di~Lorenzo and G.~Scutari, ``Next: In-network nonconvex optimization,''
  \emph{IEEE Transactions on Signal and Information Processing over Networks},
  vol.~2, no.~2, pp. 120--136, 2016.

\bibitem{vlaski2019distributed}
S.~Vlaski and A.~H. Sayed, ``Distributed learning in non-convex
  environments--part i: Agreement at a linear rate,'' \emph{arXiv preprint
  arXiv:1907.01848}, 2019.

\bibitem{vlaski2019distributed2}
------, ``Distributed learning in non-convex environments--part ii: Polynomial
  escape from saddle-points,'' \emph{arXiv preprint arXiv:1907.01849}, 2019.

\bibitem{yang2016parallel}
Y.~Yang, G.~Scutari, D.~P. Palomar, and M.~Pesavento, ``A parallel
  decomposition method for nonconvex stochastic multi-agent optimization
  problems,'' \emph{IEEE Transactions on Signal Processing}, vol.~64, no.~11,
  pp. 2949--2964, 2016.

\bibitem{scardapane2017framework}
S.~Scardapane and P.~Di~Lorenzo, ``A framework for parallel and distributed
  training of neural networks,'' \emph{Neural Networks}, vol.~91, pp. 42--54,
  2017.

\bibitem{george2019distributed}
J.~George, T.~Yang, H.~Bai, and P.~Gurram, ``Distributed stochastic gradient
  method for non-convex problems with applications in supervised learning,''
  \emph{arXiv preprint arXiv:1908.06693}, 2019.

\bibitem{akyildiz2002wireless}
I.~F. Akyildiz, W.~Su, Y.~Sankarasubramaniam, and E.~Cayirci, ``Wireless sensor
  networks: a survey,'' \emph{Computer Networks}, vol.~38, no.~4, pp. 393--422,
  2002.

\bibitem{sayed2014adaptation}
A.~H. Sayed, ``Adaptation, learning, and optimization over networks,''
  \emph{Foundations and Trends{\textregistered} in Machine Learning}, vol.~7,
  no. 4-5, pp. 311--801, 2014.

\bibitem{nedic2010constrained}
A.~Nedic, A.~Ozdaglar, and P.~A. Parrilo, ``Constrained consensus and
  optimization in multi-agent networks,'' \emph{IEEE Transactions on Automatic
  Control}, vol.~55, no.~4, pp. 922--938, 2010.

\bibitem{di2020distributed}
P.~Di~Lorenzo, S.~Barbarossa, and S.~Sardellitti, ``Distributed signal
  processing and optimization based on in-network subspace projections,''
  \emph{IEEE Transactions on Signal Processing}, vol.~68, pp. 2061--2076, 2020.

\bibitem{nassif2020adaptation}
R.~Nassif, S.~Vlaski, and A.~H. Sayed, ``Adaptation and learning over networks
  under subspace constraints—part i: Stability analysis,'' \emph{IEEE
  Transactions on Signal Processing}, vol.~68, pp. 1346--1360, 2020.

\bibitem{nassif2020adaptation2}
------, ``Adaptation and learning over networks under subspace constraints part
  ii: Performance analysis,'' \emph{IEEE Transactions on Signal Processing},
  2020.

\bibitem{ortega2018graph}
A.~Ortega, P.~Frossard, J.~Kova{\v{c}}evi{\'c}, J.~M. Moura, and
  P.~Vandergheynst, ``Graph signal processing: Overview, challenges, and
  applications,'' \emph{Proceedings of the IEEE}, vol. 106, no.~5, pp.
  808--828, 2018.

\bibitem{tsitsvero2016signals}
M.~Tsitsvero, S.~Barbarossa, and P.~Di~Lorenzo, ``Signals on graphs:
  Uncertainty principle and sampling,'' \emph{IEEE Transactions on Signal
  Processing}, vol.~64, no.~18, pp. 4845--4860, 2016.

\bibitem{sandryhaila2014discrete}
A.~Sandryhaila and J.~M. Moura, ``Discrete signal processing on graphs:
  Frequency analysis,'' \emph{IEEE Transactions on Signal Processing}, vol.~62,
  no.~12, pp. 3042--3054, 2014.

\bibitem{isufi2016autoregressive}
E.~Isufi, A.~Loukas, A.~Simonetto, and G.~Leus, ``Autoregressive moving average
  graph filtering,'' \emph{IEEE Transactions on Signal Processing}, vol.~65,
  no.~2, pp. 274--288, 2016.

\bibitem{liu2018filter}
J.~Liu, E.~Isufi, and G.~Leus, ``Filter design for autoregressive moving
  average graph filters,'' \emph{IEEE Transactions on Signal and Information
  Processing over Networks}, vol.~5, no.~1, pp. 47--60, 2018.

\bibitem{di2018adaptive}
P.~Di~Lorenzo, P.~Banelli, E.~Isufi, S.~Barbarossa, and G.~Leus, ``Adaptive
  graph signal processing: Algorithms and optimal sampling strategies,''
  \emph{IEEE Transactions on Signal Processing}, vol.~66, no.~13, pp.
  3584--3598, 2018.

\bibitem{di2016distAdaGraph}
P.~Di~Lorenzo, P.~Banelli, S.~Barbarossa, and S.~Sardellitti, ``Distributed
  adaptive learning of graph signals,'' \emph{IEEE Transactions on Signal
  Processing}, vol.~65, no.~16, pp. 4193--4208, 2017.

\bibitem{coutino2019advances}
M.~Coutino, E.~Isufi, and G.~Leus, ``Advances in distributed graph filtering,''
  \emph{IEEE Transactions on Signal Processing}, vol.~67, no.~9, pp.
  2320--2333, 2019.

\bibitem{nassif2018distributed}
R.~Nassif, C.~Richard, J.~Chen, and A.~H. Sayed, ``Distributed diffusion
  adaptation over graph signals,'' in \emph{Proc. 2018 IEEE International
  Conference on Acoustics, Speech and Signal Processing (ICASSP)}.\hskip 1em
  plus 0.5em minus 0.4em\relax IEEE, 2018, pp. 4129--4133.

\bibitem{hua2020online}
F.~Hua, R.~Nassif, C.~Richard, H.~Wang, and A.~H. Sayed, ``Online distributed
  learning over graphs with multitask graph-filter models,'' \emph{IEEE
  Transactions on Signal and Information Processing over Networks}, vol.~6, pp.
  63--77, 2020.

\bibitem{zhou2018graph}
J.~Zhou, G.~Cui, Z.~Zhang, C.~Yang, Z.~Liu, L.~Wang, C.~Li, and M.~Sun, ``Graph
  neural networks: A review of methods and applications,'' \emph{arXiv preprint
  arXiv:1812.08434}, 2018.

\bibitem{wu2019comprehensive}
Z.~Wu, S.~Pan, F.~Chen, G.~Long, C.~Zhang, and P.~S. Yu, ``A comprehensive
  survey on graph neural networks,'' \emph{arXiv preprint arXiv:1901.00596},
  2019.

\bibitem{bacciu2019gentle}
D.~Bacciu, F.~Errica, A.~Micheli, and M.~Podda, ``A gentle introduction to deep
  learning for graphs,'' \emph{arXiv preprint arXiv:1912.12693}, 2019.

\bibitem{micheli2009neural}
A.~Micheli, ``Neural network for graphs: A contextual constructive approach,''
  \emph{IEEE Transactions on Neural Networks}, vol.~20, no.~3, pp. 498--511,
  2009.

\bibitem{scarselli2008graph}
F.~Scarselli, M.~Gori, A.~C. Tsoi, M.~Hagenbuchner, and G.~Monfardini, ``The
  graph neural network model,'' \emph{IEEE Transactions on Neural Networks},
  vol.~20, no.~1, pp. 61--80, 2008.

\bibitem{defferrard2016convolutional}
M.~Defferrard, X.~Bresson, and P.~Vandergheynst, ``Convolutional neural
  networks on graphs with fast localized spectral filtering,'' in
  \emph{Advances in Neural Information Processing Systems}, 2016, pp.
  3844--3852.

\bibitem{gama2020signals}
F.~Gama, E.~Isufi, G.~Leus, and A.~Ribeiro, ``From graph filters to graph
  neural networks,'' \emph{IEEE Signal Processing Magazine}, 2020.

\bibitem{lerer2019pytorch}
A.~Lerer, L.~Wu, J.~Shen, T.~Lacroix, L.~Wehrstedt, A.~Bose, and
  A.~Peysakhovich, ``Pytorch-biggraph: A large-scale graph embedding system,''
  in \emph{Proc. Conference on Systems and Machine Learning (SysML)}, 2019.

\bibitem{zhu2019aligraph}
R.~Zhu, K.~Zhao, H.~Yang, W.~Lin, C.~Zhou, B.~Ai, Y.~Li, and J.~Zhou,
  ``Aligraph: a comprehensive graph neural network platform,''
  \emph{Proceedings of the VLDB Endowment}, vol.~12, no.~12, pp. 2094--2105,
  2019.

\bibitem{konevcny2016federated}
J.~Kone{\v{c}}n{\`y}, H.~B. McMahan, D.~Ramage, and P.~Richt{\'a}rik,
  ``Federated optimization: Distributed machine learning for on-device
  intelligence,'' \emph{arXiv preprint arXiv:1610.02527}, 2016.

\bibitem{yang2019federated}
Q.~Yang, Y.~Liu, T.~Chen, and Y.~Tong, ``Federated machine learning: Concept
  and applications,'' \emph{ACM Transactions on Intelligent Systems and
  Technology (TIST)}, vol.~10, no.~2, pp. 1--19, 2019.

\bibitem{chen2001pems}
C.~Chen, K.~Petty, A.~Skabardonis, P.~Varaiya, and Z.~Jia, ``Freeway
  performance measurement system: Mining loop detector data,''
  \emph{Transportation Research Record}, vol. 1748, no.~1, pp. 96--102, 2001.

\bibitem{velivckovic2017graph}
P.~Veli{\v{c}}kovi{\'c}, G.~Cucurull, A.~Casanova, A.~Romero, P.~Lio, and
  Y.~Bengio, ``Graph attention networks,'' in \emph{Proc. 6th International
  Conference on Learning Representations (ICLR)}, 2018.

\bibitem{xu2018powerful}
K.~Xu, W.~Hu, J.~Leskovec, and S.~Jegelka, ``How powerful are graph neural
  networks?'' in \emph{Proc. 7th International Conference on Learning
  Representations (ICLR)}, 2019.

\bibitem{isufi2020edgenets}
E.~Isufi, F.~Gama, and A.~Ribeiro, ``Edgenets: Edge varying graph neural
  networks,'' \emph{arXiv preprint arXiv:2001.07620}, 2020.

\bibitem{abadi2016deep}
M.~Abadi, A.~Chu, I.~Goodfellow, H.~B. McMahan, I.~Mironov, K.~Talwar, and
  L.~Zhang, ``Deep learning with differential privacy,'' in \emph{Proc. 2016
  ACM SIGSAC Conference on Computer and Communications Security (CCS)}, 2016,
  pp. 308--318.

\bibitem{shchur2018pitfalls}
O.~Shchur, M.~Mumme, A.~Bojchevski, and S.~G{\"u}nnemann, ``Pitfalls of graph
  neural network evaluation,'' \emph{arXiv preprint arXiv:1811.05868}, 2018.

\bibitem{dwivedi2020benchmarking}
V.~P. Dwivedi, C.~K. Joshi, T.~Laurent, Y.~Bengio, and X.~Bresson,
  ``Benchmarking graph neural networks,'' \emph{arXiv preprint
  arXiv:2003.00982}, 2020.

\bibitem{boyd2004convex}
S.~Boyd, S.~P. Boyd, and L.~Vandenberghe, \emph{Convex optimization}.\hskip 1em
  plus 0.5em minus 0.4em\relax Cambridge university press, 2004.

\bibitem{yang2016revisiting}
Z.~Yang, W.~W. Cohen, and R.~Salakhutdinov, ``Revisiting semi-supervised
  learning with graph embeddings,'' in \emph{Proc. 33rd International
  Conference on Machine Learning (ICML)}, 2016, pp. 40--48.

\bibitem{jax2018github}
\BIBentryALTinterwordspacing
J.~Bradbury, R.~Frostig, P.~Hawkins, M.~J. Johnson, C.~Leary, D.~Maclaurin, and
  S.~Wanderman-Milne, ``{JAX}: composable transformations of {P}ython+{N}um{P}y
  programs,'' 2018. [Online]. Available: \url{http://github.com/google/jax}
\BIBentrySTDinterwordspacing

\bibitem{guo2019astg}
S.~Guo, Y.~Lin, N.~Feng, C.~Song, and H.~Wan, ``Attention based
  spatial-temporal graph convolutional networks for traffic flow forecasting,''
  in \emph{Proceedings of the AAAI Conference on Artificial Intelligence},
  vol.~33, 2019, pp. 922--929.

\bibitem{yu2018spatio}
B.~Yu, H.~Yin, and Z.~Zhu, ``Spatio-temporal graph convolutional networks: a
  deep learning framework for traffic forecasting,'' in \emph{Proc. 27th
  International Joint Conference on Artificial Intelligence (IJCAI)}, 2018, pp.
  3634--3640.

\bibitem{Dan-Facch-Kung-Scut}
A.~Daneshmand, F.~Facchinei, V.~Kungurtsev, and G.~Scutari, ``Hybrid
  random/deterministic parallel algorithms for nonconvex big data
  optimization,'' \emph{IEEE Transactions on Signal Processing}, vol.~63,
  no.~13, pp. 3914--3929, Aug. 2015.

\end{thebibliography}
\end{document}